%% file: SCIS.tex
\documentclass[conference]{IEEEtran}
\IEEEoverridecommandlockouts

\usepackage{cite}
\usepackage{amsmath,amssymb,amsfonts}
\usepackage{algorithmic}
\usepackage{graphicx}
\usepackage{textcomp}
\usepackage{xcolor}
\usepackage{multirow}
\usepackage{subfigure}
\usepackage{epsfig}
\usepackage{colortbl}
\usepackage{amsmath}
\usepackage{epstopdf}
\usepackage{balance}
\usepackage[ruled,vlined]{algorithm2e}
\usepackage{hyperref}       
\usepackage{url}            

\usepackage{bbm}
\usepackage{dsfont}

\definecolor{pink}{rgb}{1,0.1,1}
\definecolor{orange}{rgb}{1,0.25,0}

\newtheorem{definition}{Definition}

\newtheorem{example}{Example}
\newtheorem{theorem}{Theorem}

\newtheorem{proposition}{Proposition}
\newenvironment{proof}{\textit{Proof}.}{\hfill$\square$}

\def\BibTeX{{\rm B\kern-.05em{\sc i\kern-.025em b}\kern-.08em
    T\kern-.1667em\lower.7ex\hbox{E}\kern-.125emX}}
\begin{document}

\title{Differentiable and Scalable Generative Adversarial Models for Data Imputation}

\author{\IEEEauthorblockN{1\textsuperscript{st} Yangyang Wu{$^{*1}$} ~ Jun Wang{$^{*2}$} ~ Xiaoye Miao{$^{*3}$} ~ Wenjia Wang{$^{*4}$}~ Jianwei Yin{$^{*\dag6}$}}
\IEEEauthorblockA{\textit{dept. name of organization (of Aff.)} \\
\textit{name of organization (of Aff.)}\\
City, Country \\
email address or ORCID}
\\
\alignauthor
 \affaddr{{$^{*}$}Center for Data Science, Zhejiang University, Hangzhou, China}\\
 \affaddr{{$^{*}$}Information Hub, The Hong Kong University of Science and Technology, Hong Kong, China}\\
 \affaddr{{$^{\dag}$}College of Computer Science, Zhejiang University, Hangzhou, China}\\
\email{zjuwuyy@zju.edu.cn, jwangfx@connect.ust.hk, miaoxy@zju.edu.cn, wenjiawang@ust.hk, zjuyjw@cs.zju.edu.cn
}
}

\author{\IEEEauthorblockN{Yangyang Wu{$^{*1}$} ~ Jun Wang{$^{\ddag2}$}~ Xiaoye Miao{$^{*3}$} ~ Wenjia Wang{$^{\ddag4}$} ~ Jianwei Yin{$^{*\dag5}$}}
\IEEEauthorblockA{{$^{*}$}Center for Data Science, Zhejiang University, Hangzhou, China}
\IEEEauthorblockA{{$^{\ddag}$}Information Hub, The Hong Kong University of Science and Technology, Hong Kong, China}
\IEEEauthorblockA{{$^{\dag}$}College of Computer Science, Zhejiang University, Hangzhou, China}
Email: \{zjuwuyy$^{1}$, miaoxy$^{3}$, zjuyjw$^{5}$\}@zju.edu.cn \qquad jwangfx$^{2}$@connect.ust.hk \qquad wenjiawang$^{4}$@ust.hk
}

\maketitle

\begin{abstract}
Data imputation has been extensively explored to solve the missing data problem.
The dramatically increasing volume of incomplete data makes the imputation models computationally infeasible in many real-life applications.
In this paper, we propose an effective scalable imputation system named \textsf{SCIS} to significantly speed up the training of the differentiable generative adversarial imputation models under accuracy-guarantees for large-scale incomplete data.
\textsf{SCIS} consists of two modules, \textit{differentiable imputation modeling} (DIM) and \textit{sample size estimation} (SSE).
DIM leverages a new \textit{masking Sinkhorn} divergence function to make an arbitrary generative adversarial imputation model differentiable, while for such a differentiable imputation model, SSE can estimate an appropriate sample size to ensure the user-specified imputation accuracy of the final model.
Extensive experiments upon several real-life large-scale datasets demonstrate that, our proposed system can accelerate the generative adversarial model training by 7.1x.
Using around 7.6\% samples, \textsf{SCIS} yields competitive accuracy with the state-of-the-art imputation methods in much shorter computation time.

\end{abstract}


\IEEEpeerreviewmaketitle

\input{1.introduction}

\input{2.related_work}

\input{3.overview}
\input{4.masking}

\input{5.estimation}

\input{6.experiment}

\input{7.conclusion}

\balance

\section*{Acknowledgment}
This work is partly supported by the Zhejiang Provincial Natural Science Foundation for Distinguished Young Scholars under Grant No.LR21F020005,
the National Natural Science Foundation of China under Grants No.61902343, No.62025206, No.61972338, No.62102351, No.61825205,
the National Key Research and Development Program of China under Grant No.2019YFE0126200,
the National Science and Technology Major Project of China under Grant No.50-D36B02-9002-16/19,
and the Fundamental Research Funds for the Central Universities under Grant No.2021FZZX001-25.
Xiaoye Miao is the corresponding author of the work.

\bibliographystyle{IEEEtran}
\bibliography{SCIS}

\end{document}

%% file: 1.introduction.tex
\section{Introduction}
\label{sec:Intr}

Many real-life scenarios, such as e-commerce, transportion science, and health care, encounter the problem of missing data~\cite{berti2018discovery,miao2021generative,rekatsinas2017holoclean,soliman2010supporting_35,wei2019embedded} as long as the data collection is involved.
Data might be missing for various reasons, including collection device failure \cite{biessmann2018deep}, instable system environment \cite{wei2019embedded}, privacy concerns \cite{cao2018brits}, etc.
The missing data problem poses a daunting challenge to the downstream data analysis.

To alleviate the missing data problem, a substantial amount of data imputation studies \cite{royston2011multiple,mattei2019miwae}
has been carried out, including the statistical ones~\cite{farhangfar2007novel}, machine learning ones \cite{stekhoven2011missforest}, multi-layer perceptron (MLP) based ones~\cite{boris2020missing}, autoencoder (AE) based ones \cite{ipsen2020not}, and generative adversarial network (GAN) based ones \cite{yoon2018gain,spinelli2019missing}.
Ideally, a preferable data imputation algorithm is to learn the true underlying data distribution well from the observed data, and count on the learned distribution to impute the missing data.
GAN-based imputation methods \cite{yoon2018gain,spinelli2019missing} are attempts in this direction.
They maximize the distributional proximity of generated data and true underlying data by introducing an adversarial game between a \emph{generator} and a \emph{discriminator}.
Many empirical investigations \cite{yoon2018gain,spinelli2019missing,kim2020survey} have demonstrated the promising performance brought by GAN-based methods for data imputation.

The ubiquity of data collection technologies with unprecedented processing power and substantial storage capacity has given rise to a dramatic increase in the volume of incomplete data.
For example, a real-world COVID-19 case surveillance public use dataset \cite{Wahltinez2020} contains 22,507,139 cases with 7 clinical and symptom features, taking a 47.62\% missing rate.
The large volume of incomplete data, however, means that it is expensive and unwieldy to exploit the above imputation algorithms.
Although GAN-based methods achieve better imputation performance than other imputation approaches, they deplete exceedingly high training cost over large-scale incomplete data.
Our experimental results show that, almost all imputation methods take more than $10^5$ seconds on training over the \emph{million}-size incomplete data.
In general, the effective and scalable data imputation over large-scale incomplete data is indispensable in many real-life scenarios.

It is challenging to apply the GAN-based imputation methods to specific real-life scenarios, particularly for large-scale incomplete data.
First, there is a strong theoretical evidence showing that, the Jensen-Shannon (JS) divergence of the GAN-based imputation model fails in the case that the true underlying and generated data distributions have non-overlapping supports\footnote{The support of a distribution is the complement of the largest open set of the random variables having zero-probability.}\cite{arjovsky2017towards,bellemare2017cramer,arjovsky2017wasserstein}.
This makes the JS divergence based imputation loss function non-differentiable, and suffer from the ``vanishing'' gradient problem.
Second, existing GAN-based imputation methods consume high training cost to calculate gradients with the mini-batch strategy for both generator and discriminator over the entire dataset.
In general, the model complexity and training sample size are two primary factors that affect the efficiency of the GAN-based imputation.

Therefore, in this paper, we propose an effective SCalable Imputation System \textsf{SCIS} to optimize GAN-based imputation models.
\textsf{SCIS} makes the GAN-based imputation model differentiable via using the optimal transport theory.
Then, for the differentiable model, it pays attention to the training sample size, and estimates an appropriate sample size for efficient and accuracy-guaranteed training.
The system is composed of two modules, a \emph{differentiable imputation modeling} (DIM) module and a \emph{sample size estimation} (SSE) module.
In terms of the first challenge, the DIM module leverages a masking Sinkhorn (MS) divergence to convert a GAN-based imputation model into a differentiable one, which can always provide reliable gradients to \emph{avoid} the ``vanishing'' gradient problem.
Regarding the second challenge, the SSE module estimates the minimum sample size to enable the trained differentiable GAN-based model to meet a user-specified error tolerance.
Hence, \textsf{SCIS} employs the minimum sample size to make the differentiable GAN-based model \emph{scalable} on large-scale incomplete data through SSE.
In summary, the main contributions of this paper are described as follows.

\begin{itemize}
\item We propose an effective scalable imputation system $\textsf{SCIS}$ with differentiable GAN-based imputation models, which can be effectively and efficiently trained.
\item In the DIM module, we put forward an MS divergence to quantify the closeness between the true underlying and generated data distributions.
    It employs the optimal transport theory to make the GAN-based imputation model differentiable, for avoiding the ``vanishing'' gradient issue.
\item The SSE module leverages the differentiability of the MS divergence to estimate the minimum sample size for training an approximate differentiable GAN-based model, according to a user-specified imputation accuracy.
    It thus significantly saves the model training cost.
\item Extensive experiments using several real-life large-scale incomplete datasets demonstrate the computational benefits of \textsf{SCIS}, compared  with the state-of-the-art methods.
\end{itemize}


The rest of the paper is organized as follows.
We introduce the background in Section~\ref{sec:background}.
Section~\ref{sec:overview} gives an overview of the proposed system \textsf{SCIS}.
We elaborate the DIM and SSE modules in Section~\ref{sec:DIM} and Section~\ref{sec:SSE}, respectively.
Section~\ref{sec:experiment} reports the experimental results and findings.
Finally, we conclude this work in Section~\ref{sec:conclusion}.

%% file: 2.related_work.tex
\section{Background}
\label{sec:background}

\subsection{Existing Imputation Methods}
\label{subsec:Related-Work}

Existing imputation algorithms can be categorized by their mainly used models, including statistical ones, machine learning ones, and deep learning ones.
The statistical imputation methods substitute the missing values with the statistics~\cite{farhangfar2007novel}, or the most similar ones among the training data~\cite{twala2005comparison,altman1992introduction}.
This kind of method has a limited ability to handle the missing data problem, since they ignore the data distribution analysis.

In contrast, the machine learning imputation approaches are to train parametric models in machine learning \cite{zhang2021parrot,yang2021keyword} to estimate the missing values,
including decision tree models like XGBoost imputation \cite{chen2016xgboost}, MissFI (MissForest imputation) \cite{stekhoven2011missforest}, and Baran \cite{mahdavi2020baran}, and regression models like MICE (multivariate imputation by chained equations) \cite{royston2011multiple} and imputation via individual model~\cite{zhang2019learning}.
These imputation methods employ the batch gradient descent techniques \cite{ruder2016overview} to calculate the model gradient over the entire dataset.
The main issue is that, the incomplete dataset may be too large to fit in memory.

Moreover, the deep learning imputation methods leverage deep learning models \cite{du2021deep,mahajan2021predicting} to impute missing values with the mini-batch gradient descent.
This category consists of i) MLP-based ones like DataWig \cite{biessmann2019datawig} and RRSI (round-robin Sinkhorn imputation)~\cite{boris2020missing}, ii) AE-based ones, e.g., MIDAE (multiple imputation denoising autoencoder) \cite{gondara2017multiple}, VAEI (variational autoencoder imputation) \cite{mccoy2018variational}, EDDI (efficient dynamic discovery of highvalue information framework) \cite{ma2018eddi}, HIVAE (heterogeneous incomplete variational autoencoder) \cite{nazabal2018handling}, MIWAE (missing data importance-weighted autoencoder) \cite{mattei2019miwae}, and not-MIWAE (not-missing-at-random importance-weighted autoencoder) \cite{ipsen2020not}, and iii) GAN-based ones, such as GINN (graph imputation neural network) \cite{spinelli2019missing} and GAIN (generative adversarial imputation network) \cite{yoon2018gain}.
The above methods calculate the model gradients with a series of random  partitions of the dataset, to train the imputation models over large-scale incomplete data.
Nevertheless, both the iteration times and training cost of these methods are dramatically increasing with the rising volume of incomplete data.

As analyzed earlier, the data generated by GAN-based methods tend to be closer to the true underlying data manifold than that of the other ones.
Nevertheless, due to the high training cost and ``vanishing'' gradient problem, the GAN-based imputation methods are faced with a big challenge to deal with large-scale incomplete data imputation.
In this paper, our proposed system \textsf{SCIS} aims to provide good effectiveness and scalability for GAN-based imputation models, so as to make them practical in real-world applications.

\begin{table}[t]
\centering\small
\setlength{\tabcolsep}{2pt}
\caption{Symbols and description}
\vspace*{-0.08in}
\label{tab:symbol}
\begin{tabular}{|p{1.6cm}|p{6.5cm}|}
\hline
\textbf{Symbol}           &  \textbf{Description} \\ \hline
$\mathbf{X}$    &   an input incomplete dataset (stored in a matrix) \\\hline
$\mathbf{M}$     &   the mask matrix w.r.t. $\mathbf{X}$ \\\hline
$\mathcal{M}$ and $\mathcal{M}_\star$     &  an imputation model and the optimized model \\ \hline
$\mathbf{\bar{X}}$ and $\mathbf{\hat{X}}$&  the reconstructed matrix and imputed matrix of $\mathbf{X}$ \\ \hline
$\mathbf{X}_0$ and $\mathbf{M}_0$     &   the initial dataset matrix and its mask matrix\\\hline
$\mathbf{X}_v$ and $\mathbf{M}_v$     &   the validation dataset matrix and its mask matrix  \\\hline
$n_0$ and $N_v$     &   the initial sample size and validation sample size\\\hline
$\varepsilon$ and $\alpha$ &  the user-tolerated error bound and confidence level \\ \hline
\end{tabular}
\vspace*{-0.1in}
\end{table}

\subsection{Problem Definition}

The input incomplete dataset is stored in a matrix $\mathbf{X} = (\mathbf{x}_{1}, \cdots, \mathbf{x}_{N})^\top$, in which each data sample $\mathbf{x}_i = (x_{i1}, \cdots,$ $x_{id})^{}$ with $x_{ij} \in \mathcal{X}_{j}$ takes values from a  $d-$dimensional space $\mathcal{X} = \mathcal{X}_{1} \times \dots \times \mathcal{X}_{d}$.
For encoding its missing information, we define a mask matrix $\mathbf{M} = (\mathbf{m}_1, \cdots, \mathbf{m}_N)^\top$,  where each mask vector $\mathbf{m}_i=(m_{i1}, \cdots,$ $m_{id})$ corresponds to a data sample $\mathbf{x}_i$.
In particular, $m_{ij}$ takes value from $\{0, 1\}$, $i = 1, \cdots, s$, and $j = 1, \cdots, d$; $m_{ij}$ = 1 (resp. 0) \emph{iff} the $j$-th dimension is observed (resp. missing).
Note that, we use $\mathbf{X}$ and $\mathbf{X}_N$ interchangeably throughout this paper.
Table~\ref{tab:symbol} lists the frequently used symbols throughout this paper.

In general, we formalize the \emph{data imputation} problem over $\mathbf{X}$ and $\mathbf{M}$ in Definition \ref{def:mdi}.

\begin{definition}
\label{def:mdi}
{\bf (Data imputation)}.
Given an incomplete dataset $\mathbf{X}$ with its mask matrix $\mathbf{M}$, the data imputation problem aims to build an imputation model $\mathcal{M}$ to find appropriate values for missing values in $\mathbf{X}$, i.e., the \emph{imputed} matrix
\begin{equation}
\begin{aligned}
\mathbf{\hat{X}} = \mathbf{{M}} \odot \mathbf{{X}} + (1 -\mathbf{M}) \odot \mathbf{\bar{X}}
\end{aligned}
\label{eq:imputation}
\end{equation}
where $\odot$ is the element-wise multiplication; $\mathbf{\bar{X}} = \mathcal{M}(\mathbf{{X}})$ is the reconstructed matrix predicted by $\mathcal{M}$ over $\mathbf{{X}}$. In this way, the imputation model $\mathcal{M}$, (i) makes $\mathbf{\hat{X}}$ as \emph{close} to the real \emph{complete} dataset (if it exists) as possible, or (ii) helps downstream prediction tasks to achieve better performance if adopting $\mathbf{\hat{X}}$ than that only with original one $\mathbf{X}$.
\end{definition}

In this paper, our study mission is to empower the GAN-based imputation model with efficiency and effectiveness for large-scale incomplete data, such that for the optimized model, (i) the training cost is minimized, and (ii) the imputation accuracy satisfies a user-tolerated error bound.

In particular, the GAN-based imputation model builds an adversarial training platform for two players, i.e., the generator and discriminator.
The generator is applied to generate data as close to the true underlying data distribution as possible. While the discriminator distinguishes the difference between the generated data and true underlying data as correctly as possible.
The objective function of GAN-based imputation is defined as a minimax problem over the generator and discriminator.
As a result, the GAN-based model employs the optimized generator to impute missing values via using Eq. \ref{eq:imputation}.

%

%% file: 3.overview.tex
\section{System Overview}
\label{sec:overview}

In this section, we present an overview of the proposed system \textsf{SCIS}. It consists of a \emph{differentiable imputation modeling} (DIM) module and a \emph{sample size estimation} (SSE) module.

To facilitate the effective and scalable imputation on large-scale incomplete data, our proposed system \textsf{SCIS} first converts a GAN-based imputation model into a differentiable one, to avoid the ``vanishing'' gradient problem.
Then, for such a differentiable GAN-based imputation model, \textsf{SCIS} minimizes the training sample size under accuracy-guarantees to accelerate the imputation.
Its general procedure is described in Algorithm \ref{alg:SCIS}.
At first, \textsf{SCIS} samples a size-$N_v$ validation dataset $\mathbf{X}_v\in \mathbb{R}^{N_v\times d}$ (with the validation mask matrix $\mathbf{M}_v$) from the incomplete input dataset $\mathbf{X}\in \mathbb{R}^{N\times d}$ (with the mask matrix $\mathbf{M}$), while it samples a size-$n_0$ dataset $\mathbf{X}_0\in \mathbb{R}^{n_0\times d}$ (with the initial mask matrix $\mathbf{M}_0$) from the rest (line 1).
Then, \textsf{SCIS} invokes the DIM module to train an initial model $\mathcal{M}_0$ with the masking Sinkhorn (MS) divergence imputation loss function over $\mathbf{X}_0$ and $\mathbf{M}_0$ (line 2).

Next, with the support of the differentiablity of MS divergence, \textsf{SCIS} consults the SSE module to estimate the minimum sample size $n_\star$ (with $n_0\le n_\star\le N$) based on $\mathcal{M}_0$, such that the imputation difference of the trained models over the size-$n_\star$ dataset and the (whole) size-$N$ dataset would not exceed the user-tolerated error bound $\varepsilon$ with probability at least ($1-\alpha$).
In particular, if $n_\star$ is equal to $n_0$, \textsf{SCIS} simply outputs $\mathcal{M}_0$ and the matrix $\hat{\mathbf{X}}$ imputed by $\mathcal{M}_0$.
Otherwise, \textsf{SCIS} constructs a size-$n_\star$ sample set $\mathbf{X}_{\star}$ and its mask matrix $\mathbf{M}_{\star}$ from $\mathbf{X}$ and $\mathbf{M}$.
It invokes the DIM module again to retrain $\mathcal{M}_0$ using $\mathbf{X}_{\star}$ and $\mathbf{M}_{\star}$ (lines 3-5).
Finally, \textsf{SCIS} returns the trained model $\mathcal{M}_\star$ and the matrix $\hat{\mathbf{X}}$ imputed by $\mathcal{M}_\star$.

\begin{algorithm}[t]
\caption{The Procedure of \textsf{SCIS}}
\label{alg:SCIS}
\DontPrintSemicolon
\LinesNumbered
\SetNlSty{large}{}{:}
\KwIn{an incomplete set $\mathbf{X}$ with its mask matrix $\mathbf{M}$, a validation size $N_v$, an initial size $n_0$, a GAN-based model $\mathcal{M}$, a user-tolerated error bound $\varepsilon$, and a confidence level $\alpha$}
\KwOut{the trained model $\mathcal{M}_\star$ and imputed data $\hat{\mathbf{X}}$}
sample a size-$N_v$ validation dataset $\mathbf{X}_v$ (with $\mathbf{M}_v$) and a size-$n_0$ initial dataset $\mathbf{X}_0$ (with $\mathbf{M}_0$)\;
invoke DIM to train an initial model $\mathcal{M}_0$ with the MS divergence loss function over $\mathbf{X}_0$ and $\mathbf{M}_0$\;
consult SSE to derive the minimum size $n_\star$ to satisfy the error bound $\varepsilon$ with probability at least ($1-\alpha$)\;
\lIf{$n_\star = n_0$}{$\mathcal{M}_\star\gets \mathcal{M}_0$}
\lElse{invoke DIM to train $\mathcal{M}_0$ on the minimum sample set $\mathbf{X}_{\star}$ and $\mathbf{M}_{\star}$ to obtain the optimized $\mathcal{M}_\star$}
reconstruct $\mathbf{X}$ via using $\mathcal{M}_\star$ to obtain $\bar{\mathbf{X}}$\;
$\hat{\mathbf{X}} = \mathbf{M} \odot \mathbf{X} + (1 - \mathbf{M}) \odot \bar{\mathbf{X}}$\;
\Return $\mathcal{M}_\star$ and $\hat{\mathbf{X}}$
\end{algorithm}

%% file: 4.masking.tex
\section{Differentiable Imputation Modeling}
\label{sec:DIM}

For the GAN-based imputation model, the intersection in the supports of the true underlying and generated data distributions is usually negligible \cite{arjovsky2017towards}.
In such case, the JS divergence makes the GAN-based imputation models non-differentiable, and even suffer from the ``vanishing'' gradient problem.
This problem may prevent the model parameter from updating its value or even stop model from further training.

In the \textsf{SCIS}, the \emph{differentiable imputation modeling} (DIM) module designs a masking Sinkhorn (MS) divergence by deploying the mask matrix from the imputation task and optimal transport theory, to make the GAN-based imputation models \emph{differentiable}, and thus obtain reliable gradients through the model training and circumvent ``vanishing'' gradient problem.
In this section, we first introduce the MS divergence imputation loss function.
Then, we elaborate how to optimize the GAN-based imputation model via the MS divergence.

\subsection{Masking Sinkhorn Divergence}
\label{subsec:Sinkhorn}

The differentiable masking Sinkhorn (MS) divergence is devised by introducing the mask matrix, entropic regularizer, and corrective terms into the optimal transport metric.
Specifically, in the MS divergence, we first empower the optimal transport metric with the mask matrix, to devise the masking optimal transport metric for data imputation.
Then, we put forward the entropic regularizer to make the masking regularized optimal transport metric differentiable and programmable.
We also equip the MS divergence with corrective terms to correct the bias from the entropic regularizer.

To be more specific, the masking optimal transport metric is introduced by deploying the mask matrix and optimal transport metric, as stated in Definition \ref{def:mot}.
It is based on a simple intuition that the observed and generated data should share the same distribution.
Let $\hat{\mu}_{\mathbf{x}} \overset{def}{=} \frac{1}{n}\sum_{i=1}^{n} \delta_{\mathbf{x}_{i}}$, $\hat{\mu}_{\mathbf{m}} \overset{def}{=} \frac{1}{n}\sum_{i=1}^{n} \delta_{\mathbf{m}_{i}}$, and  $\hat{\nu}_{\bar{\mathbf{x}}} \overset{def}{=} \frac{1}{n}\sum_{i=1}^{n} \delta_{\bar{\mathbf{x}}_{i}}$ denote the empirical measures over a size-$n$ data matrix $\mathbf{X}_n\subset\mathbf{X}$, its mask matrix $\mathbf{M}_n$, and the reconstructed matrix $\bar{\mathbf{X}}_n$ respectively, where
$\delta_{\mathbf{x}_i}$, $\delta_{\mathbf{m}_i}$, and $\delta_{\bar{\mathbf{x}}_{i}}$ are the Dirac distributions.

\begin{definition}
\label{def:mot}
{\bf (The masking optimal transport)}.
The masking optimal transport metric $OT^{\mathbf{m}} $ $ (\hat{\nu}_{\bar{\mathbf{x}}},$ $\hat{\mu}_{\mathbf{x}})$ over $\hat{\nu}_{\bar{\mathbf{x}}}$ and $\hat{\mu}_{\mathbf{x}}$ is defined as the optimal transport metric $OT(\hat{\nu}_{\bar{\mathbf{x}}} \otimes \hat{\mu}_{\mathbf{m}},  \hat{\mu}_{\mathbf{x}} \otimes \hat{\mu}_{\mathbf{m}})$ over $\hat{\nu}_{\bar{\mathbf{x}}} \otimes \hat{\mu}_{\mathbf{m}}$ and $\hat{\mu}_{\mathbf{x}} \otimes \hat{\mu}_{\mathbf{m}}$, i.e.,
\begin{equation}\nonumber
\begin{aligned}
      OT^{\mathbf{m}} (\hat{\nu}_{\bar{\mathbf{x}}}, \hat{\mu}_{\mathbf{x}}) &= OT(\hat{\nu}_{\bar{\mathbf{x}}} \otimes \hat{\mu}_{\mathbf{m}},  \hat{\mu}_{\mathbf{x}} \otimes \hat{\mu}_{\mathbf{m}}) = \min_{\mathbf{P} \in \Gamma_{n, n}} \langle \mathbf{P} , \mathbf{C}_\mathbf{m}\rangle,
\label{eq:mot}
\end{aligned}
\end{equation}
where $\hat{\nu}_{\bar{\mathbf{x}}} \otimes \hat{\mu}_{\mathbf{m}}$ (resp. $\hat{\mu}_{\mathbf{x}} \otimes \hat{\mu}_{\mathbf{m}}$) stands for the product measure of $\hat{\nu}_{\bar{\mathbf{x}}}$ (resp. $\hat{\mu}_{\mathbf{x}}$) and $\hat{\mu}_{\mathbf{m}}$.
The transportation plan matrix $\mathbf{P}$ is from the set $\Gamma_{n, n} \overset{def}{=} \{\mathbf{P} \in \mathbb{R}^{n\times n}: \mathbf{P}\mathds{1}_{n} = \frac{1}{n}\mathds{1}_{n}, \mathbf{P}^{\top}\mathds{1}_{n} = \frac{1}{n}\mathds{1}_{n}\}$.
The masking cost matrix $\mathbf{C}_\mathbf{m}$ is defined as
\begin{equation} \nonumber
\begin{aligned}
\mathbf{C}_\mathbf{m} = \big(f_c(\mathbf{m}_{i} \odot \bar{\mathbf{x}}_{i}, \mathbf{m}_{j} \odot \mathbf{x}_{j})\big)_{ij} \in \mathbb{R}^{n\times n},
\end{aligned}
\end{equation}
where $\odot$ is the element-wise multiplication; $f_c(\mathbf{x}, \mathbf{y}) = \vert|\mathbf{x} - \mathbf{y}\vert|^2_2$ is the cost function;
$\langle \mathbf{P}, \mathbf{C}_\mathbf{m}\rangle = {\rm tr}(\mathbf{P}^{\top}\mathbf{C}_\mathbf{m})$ is the Frobenius dot-product of $\mathbf{P}$ and $\mathbf{C}_\mathbf{m}$.
\end{definition}

Nervertheless, the masking optimal transport metric is still not differentiable \cite{peyre2019computational}.
Moreover, there exists a costly linear program for computing the transport plan $\mathbf{P}$ in the masking optimal transport metric.
As a result, we further introduce a masking regularized optimal transport metric, as stated in Definition \ref{def:mrot}. This metric becomes differentiable and programmable through an entropic regularization term \cite{genevay2016stochastic}.

\begin{definition}
\label{def:mrot}
{\bf (The masking regularized optimal transport)}.
The masking regularized optimal transport metric $OT^{\mathbf{m}}_{\lambda} $ $ (\hat{\nu}_{\bar{\mathbf{x}}},$ $\hat{\mu}_{\mathbf{x}})$ over $\hat{\nu}_{\bar{\mathbf{x}}}$ and $\hat{\mu}_{\mathbf{x}}$ is defined as
\begin{equation}
\begin{aligned}
      OT^{\mathbf{m}}_{\lambda} (\hat{\nu}_{\bar{\mathbf{x}}}, \hat{\mu}_{\mathbf{x}})
      &= OT^{\mathbf{m}} (\hat{\nu}_{\bar{\mathbf{x}}}, \hat{\mu}_{\mathbf{x}})  + \lambda h(\mathbf{P}) \\
      &=  \min_{\mathbf{P} \in \Gamma_{n, n}} \langle \mathbf{P} , \mathbf{C}_\mathbf{m}\rangle + \lambda \sum_{i=1}^{n}\sum_{j=1}^{n}p_{ij}\log{p_{ij}}
\label{eq:mot}
\end{aligned}
\end{equation}
where $\lambda$ is a hyper-parameter.
\end{definition}

Due to the entropic regularizer $h(\mathbf{P})$ in Eq.~\ref{eq:mot}, the resolution of the optimal transport plan $\mathbf{P}^{\star}$ can be tractable by using Sinkhorn's algorithm~\cite{sinkhorn1964relationship}, and thus the masking regularized optimal transport metric becomes algorithmically differentiable and easily programmable.
This entropic regularizer makes the masking regularized optimal transport metric no longer positive.
This may make the imputation models focus too much on learning the mean of the true underlying data distribution (i.e., ignoring the whole distribution).
To get rid of this issue, we add corrective terms into the MS divergence to assure itself of positivity.

\begin{definition}
\label{def:msd}
{\bf (The masking Sinkhorn divergence)}.
The masking Sinkhorn (MS) divergence $\mathcal{S}_{\mathbf{m}}(\hat{\nu}_{\bar{\mathbf{x}}} ||  \hat{\mu}_{\mathbf{x}})$ between the empirical measures $\hat{\nu}_{\bar{\mathbf{x}}}$ and $\hat{\mu}_{\mathbf{x}}$ is defined as
\begin{equation}\nonumber
\begin{aligned}
    \label{imploss}
      \mathcal{S}_{\mathbf{m}}(\hat{\nu}_{\bar{\mathbf{x}}} ||  \hat{\mu}_{\mathbf{x}}) =& 
      2OT^{\mathbf{m}}_{\lambda} (\hat{\nu}_{\bar{\mathbf{x}}}, \hat{\mu}_{\mathbf{x}}) \\&- [OT^{\mathbf{m}}_{\lambda} (\hat{\nu}_{\bar{\mathbf{x}}}, \hat{\nu}_{\bar{\mathbf{x}}}) + OT^{\mathbf{m}}_{\lambda} (\hat{\mu}_{\mathbf{x}}, \hat{\mu}_{\mathbf{x}})].
\end{aligned}
\end{equation}
\end{definition}

Additionally, we instantiate how the MS divergence handles the ``vanishing'' gradient problem in the following example. Hereby, we also show the non-differentiable property of the JS divergence and the differentiable property of the MS divergence.
Throughout this paper, we make the assumption that the data are missing
completely at random (i.e., MCAR).

\begin{example}
For the imputation task defined in Definition \ref{def:mdi}, we consider $d = 1$ and $\mathcal{X}_{1} = \mathbb{R}$.
Under MCAR mechanism, the mask vector $\mathbf{m} \in \{0, 1\}$ is independent of the sample $\mathbf{x} \in \mathcal{X}_{1}$, i.e., $p_{m}(\mathbf{m}|\mathbf{x}) = p_{m}(\mathbf{m})$.
Thus, a joint distribution of $\mathbf{x}$ and $\mathbf{m}$ can be defined as $p(\mathbf{x}, \mathbf{m}) = p_{x}(\mathbf{x})p_{m}(\mathbf{m})$.
In particular, the true underlying and generated data distributions are defined as ${p}^{0}_{x}(\mathbf{x}) = \delta_{0}$ and ${p}^{\theta}_{x}(\mathbf{x}) = \delta_{\theta}$ with the parameter $\theta \in \mathbb{R}$, respectively.
Besides, the missingness distribution $p_{m}(\mathbf{m})$ is supposed to follow a \textit{Bernoulli} distribution $Ber(q)$ with a constant $ q \in (0, 1)$.
For simplicity, we denote $p_{0}$ and $p_{\theta}$ as the distributions of $p^{0}_{x}(\mathbf{x})p_{m}(\mathbf{m})$ and $p^{\theta}_{x}(\mathbf{x})p_{m}(\mathbf{m})$, respectively.
Thus, the JS divergence between ${p}_{0}$ and ${p}_{\theta}$ is calculated by
\begin{equation*}
    \begin{aligned} \nonumber
JS&({p}_{0} || {p}_{\theta} ) = \int  {p}^{0}_{x}(\mathbf{x})q \log \frac{2{p}^{0}_{x}(\mathbf{x})q}{{p}^{0}_{x}(\mathbf{x})q+{p}^{\theta}_{x}(\mathbf{x})q} d \mathbf{x} \\ \nonumber
& + \int  {p}^{0}_{x}(\mathbf{x})(1 - q) \log \frac{2{p}^{0}_{x}(\mathbf{x})(1 - q)}{{p}^{0}_{x}(\mathbf{x})(1 - q) + {p}^{\theta}_{x}(\mathbf{x})(1 - q)} d \mathbf{x} \\ \nonumber
&+ \int  {p}^{\theta}_{x}(\mathbf{x})q \log \frac{2{p}^{\theta}_{x}(\mathbf{x})q}{{p}^{0}_{x}(\mathbf{x})q+{p}^{\theta}_{x}(\mathbf{x})q} d \mathbf{x}\\
&+ \int  {p}^{\theta}_{x}(\mathbf{x})(1 - q) \log \frac{2{p}^{\theta}_{x}(\mathbf{x})(1 - q)}{{p}^{0}_{x}(\mathbf{x})(1 - q) + {p}^{\theta}_{x}(\mathbf{x})(1 - q)} d \mathbf{x}\\
 &= \begin{cases}
0  & \mbox{if }\theta = 0 \\ \nonumber
2\log 2 & \mbox{else}.
\end{cases}
\end{aligned}
\end{equation*}
We can find that, $JS({p}_{0} || {p}_{\theta})$ is not continuous at $\theta = 0$, and thus non-differentiable.
Moreover, the gradients of $JS({p}_{0} || {p}_{\theta} )$ are 0 almost everywhere.
Thus, $JS({p}_{0} || {p}_{\theta})$ provides useless gradient information to update the model parameter $\theta$, which underlies the ``vanishing'' gradient problem.

In contrast, the differentiable MS divergence between ${p}_{0}$ and ${p}_{\theta}$ is calculated by
\begin{equation*}
\begin{aligned}
\mathcal{S}_{\mathbf{m}}(p_{0}, p_{\theta})
=&  2OT^{\mathbf{m}}_{\lambda} (p_{0}, p_{\theta}) \\& - [OT^{\mathbf{m}}_{\lambda} (p_{0}, p_{0}) + OT^{\mathbf{m}}_{\lambda} (p_{\theta}, p_{\theta})],
\end{aligned}
\end{equation*}
in particular, $OT^{\mathbf{m}}_{\lambda} (p_{0}, p_{0}) = 0$.
$OT^{\mathbf{m}}_{\lambda} (p_{0}, p_{\theta})$ is calculated by
\begin{equation*}
\begin{aligned}
OT^{\mathbf{m}}_{\lambda} (p_{\theta}, p_{\theta}) = \lambda[(1 - q)\log(1 - q) + q\log q]
\end{aligned}
\end{equation*}
Moreover, $OT^{\mathbf{m}}_{\lambda} (p_{0}, p_{\theta})$ is defined as
\begin{equation*}
\begin{aligned}
OT^{\mathbf{m}}_{\lambda} (p_{0}, p_{\theta}) \stackrel{\text { def }}{=} & \min_{\gamma \in \Pi(p_{0}, p_{\theta})} \mathbb{E}_{(\mathbf{y}, \mathbf{y}')\sim \gamma}[(\mathbf{y} - \mathbf{y}')^2 \\&+ \lambda \log \gamma(\mathbf{y},\mathbf{y}')]\\
= & \quad q\theta^{2} + \lambda[(1 - q)\log(1 - q) + q\log q],
\end{aligned}
\end{equation*}
where $\Pi(p_{0}, p_{\theta})$ denotes the set of all joint distributions $\gamma(\mathbf{y}, \mathbf{y}')$, whose marginals are $p_{0}$ and $p_{\theta}$, respectively. The second equality exploits the fact that the optimal joint distribution $\gamma^{*}(\mathbf{y}, \mathbf{y}')$ is calculated by
\begin{equation*}
    \gamma^{*}(\mathbf{y}, \mathbf{y}') = \begin{cases}
1 - q  & \mbox{if } \mathbf{y} = 0 \mbox{ and } \mathbf{y}' = 0  \\ \nonumber
q  & \mbox{if } \mathbf{y} = 0  \mbox{ and } \mathbf{y}' = \theta  \\ \nonumber
0 & \mbox{else}.
\end{cases}
\end{equation*}
Thus, $\mathcal{S}_{\mathbf{m}}(p_{0}, p_{\theta}) = 2q\theta^{2} + \lambda[(1 - q)\log(1 - q) + q\log q]$.
It is obvious that, $\mathcal{S}_{\mathbf{m}}(p_{0}, p_{\theta})$ is differentiable everywhere with respect to $\theta$.
The gradients of $\mathcal{S}_{\mathbf{m}}(p_{0}, p_{\theta})$ vary linearly, which can always provide reliable gradient information to update $\theta$, and thus dispose of the ``vanishing'' gradient problem.
\end{example}

In general, the differentiable MS divergence imputation loss function can be defined as
\begin{equation}\nonumber
\begin{aligned}
    \label{eq: SDILoss}
      \mathcal{L}_s (\mathbf{X}, \mathbf{M}) 
      =\frac{1}{2n}\mathcal{S}_{\mathbf{m}}(\hat{\nu}_{\bar{\mathbf{x}}} ||  \hat{\mu}_{\mathbf{x}}).
\end{aligned}
\end{equation}
By virtue of the \emph{differentiable} MS divergence, the MS divergence imputation loss function can provide a usable and reliable gradient during the training of the GAN-based model, and thus it helps to get rid of the ``vanishing'' gradient issue.

It is worthwhile to note that, the proposed MS divergence from \textsf{SCIS} is to minimize the MS divergence between the true underlying data distribution in the observed data and generated data distribution for GAN-based imputation models.
Hence, with the support of the MS divergence, the true underlying data distribution can be well preserved in the imputed data from observed values.
It is different from the Sinkhorn divergence used in the round-robin Sinkhorn imputation algorithm (RRSI) \cite{boris2020missing}, where the Sinkhorn divergence between any two imputed batches is minimized for the MLP-based imputation model.
In essence, the data distribution predicted by RRSI easily converges to a mixed distribution of the observed data and initial imputed data, rather than the true underlying one, especially when there exist a large amount of missing data.
More intuitively, at the initial step, RRSI fills the missing data by the mean value of the observed data for each incomplete feature separately.
Thus, RRSI is likely to regard two samples with similar missing patterns as similar ones, while these two samples are actually different in the ground true space.

%
%

\subsection{Imputation Optimization with MS Divergence}
\label{subsec:gradient}

In DIM module, we convert a GAN-based imputation model into a differentiable one, by taking the MS divergence to measure the closeness between the true underlying data distribution in $\mathbf{X}_n$ and the generated data distribution in $\bar{\mathbf{X}}_n$.
In pursuit of deriving the GAN-based imputation model to predict $\bar{\mathbf{X}}_n$ as similar to the observed data in $\mathbf{X}_n$ as possible, our goal becomes to find the optimized parameter $\hat{\theta}$ minimizing $\mathcal{S}_{\mathbf{m}}(\hat{\nu}_{\bar{\mathbf{x}}} ||  \hat{\mu}_{\mathbf{x}})$ from a parameter space $\Theta$. Formally, we rewrite $\mathcal{S}_{\mathbf{m}}(\hat{\nu}_{\bar{\mathbf{x}}} ||  \hat{\mu}_{\mathbf{x}})$
as $\mathcal{S}_{\mathbf{m}}(\bar{\mathbf{X}}_n \odot \mathbf{M}_n, {\mathbf{X}}_n\odot \mathbf{M}_n)$. Therefore, the MS divergence imputation loss minimizer is given by
\begin{equation}
\label{eq:obj1}
\begin{aligned}
\hat{\theta}
= \mathop{\arg\min}_{\theta \in \Theta} \frac{1}{2n}\mathcal{S}_{\mathbf{m}}(\bar{\mathbf{X}}_n \odot \mathbf{M}_n, {\mathbf{X}}_n\odot \mathbf{M}_n).
\end{aligned}
\end{equation}

By using the chain rule and the barycentric transport map \cite{cuturi2014fast}, the MS divergence gradient function can be derived by the following proposition.

\begin{proposition}
The MS divergence gradient function $g(\theta)$ can be calculated by
\begin{equation}
\begin{aligned} \nonumber
    g(\theta)&=\frac{1}{2n}\nabla_{\theta}\mathcal{S}_{\mathbf{m}}(\bar{\mathbf{X}}_{n}\odot\mathbf{M}_{n}, \mathbf{X}_{n} \odot \mathbf{M}_{n})\\
  &=\frac{1}{2n}\nabla_{\bar{\mathbf{X}}_{n}}\mathcal{S}_{\mathbf{m}}(\bar{\mathbf{X}}_{n}\odot\mathbf{M}_{n}, \mathbf{X}_{n} \odot \mathbf{M}_{n})\nabla_{\theta}\bar{\mathbf{X}}_{n}\\
    &=\frac{1}{n}\sum_{j=1}^{n} \big[\sum_{i=1}^{n}\mathbf{P}^{\star}_{ij} (\bar{\mathbf{\mathbf{x}}}_{i} \odot \mathbf{m}_{i} - {\mathbf{x}}_{j}\odot\mathbf{m}_{j})  \mathcal{T}(\mathbf{m}_{i})  \nabla_{\theta} \bar{\mathbf{x}}_{i} \big], 
\label{eq:paramchange}
\end{aligned}
\end{equation}
where $\nabla_{\theta}\mathcal{S}_{\mathbf{m}}(\cdot)$ and $\nabla_{\bar{\mathbf{X}}_n}\mathcal{S}_{\mathbf{m}}(\cdot)$ are the derivatives of $\mathcal{S}_{\mathbf{m}}(\cdot)$ with respect to the parameter $\theta$ and the reconstructed matrix $\bar{\mathbf{X}}_n$, respectively;
$\nabla_{\theta}\bar{\mathbf{X}}_n$ is the derivative of $\bar{\mathbf{X}}_n$ with respect to $\theta$;
 $\mathbf{P}^{\star} = (\mathbf{P}^{\star}_{ij})_{ij}$ is the optimal transport plan;
$\mathcal{T}(\mathbf{m}_{i})$ is to transform a mask vector $\mathbf{m}_{i}$ to a diagonal matrix.
\end{proposition}

\begin{proof}
First, by using the barycentric transport map, we obtain that $\forall i \in {1,\dots,n}$,
\begin{align*}
    \nabla{\bar{\mathbf{x}}_{i}}OT^{\mathbf{m}}_{\lambda}&(\bar{\mathbf{X}}_{n} \odot \mathbf{M}_{n}, {\mathbf{X}}_{n} \odot \mathbf{M}_{n})  \\
    &= \bigg[\sum_{j=1}^{n}\mathbf{P}^{\star}_{ij} (\bar{\mathbf{x}}_{i} \odot \mathbf{m}_{i}- \mathbf{x}_{j} \odot \mathbf{m}_{j})\bigg]  {\mathcal{T}}(\mathbf{m}_i). 
\end{align*}
Then, by using the chain rule, the gradient of $\mathcal{L}_{s}(\mathbf{X}_{n}, \mathbf{M}_{n})$ can be calculated by
\begin{align*}
    g(\theta)=&
    \left[\frac{1}{2n}\nabla_{\bar{\mathbf{X}}_{n}}\mathcal{S}_{m}( \bar{\mathbf{X}}_{n} \odot \mathbf{M}_{n},{\mathbf{X}}_{n} \odot \mathbf{M}_{n})\right]\nabla_{\theta}\bar{\mathbf{X}}\\
    =&\frac{1}{n}\sum_{i=1}^{n}\bigg[\sum_{j=1}^{n}\mathbf{P}^{\star}_{ij}(\bar{\mathbf{x}}_{i} \odot \mathbf{m}_{i}  - {\mathbf{x}}_{j} \odot \mathbf{m}_{j})\bigg] \mathcal{T}(\mathbf{m}_i)  \nabla_{\theta}\bar{\mathbf{x}}_{i}\\
    =&\frac{1}{n}\sum_{j=1}^{n}\bigg[\sum_{i=1}^{n}\mathbf{P}^{\star}_{ij}(\bar{\mathbf{x}}_{i} \odot  \mathbf{m}_{i} - {\mathbf{x}}_{j} \odot \mathbf{m}_{j}) {\mathcal{T}}(\mathbf{m}_i)  \nabla_{\theta}\bar{\mathbf{x}}_{i}\bigg].
\end{align*}
\end{proof}

\begin{figure}[t]
\center
  \includegraphics[width=1.05\linewidth]{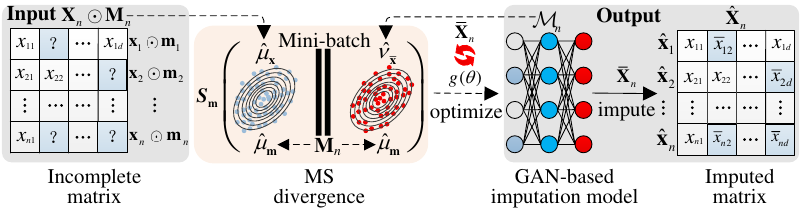}
  \vspace*{-0.3in}
  \caption{The architecture of the DIM module}
 \label{fig:DIM}
 \vspace*{-0.2in}
\end{figure}

Consequently, as depicted in Figure~\ref{fig:DIM}, the DIM module takes the data matrix $\mathbf{X}_n$ and its mask matrix $\mathbf{M}_n$ as inputs, and outputs the optimized GAN-based imputation model $\mathcal{M}_{n}$ and the data matrix $\hat{\mathbf{X}}_n$ imputed by $\mathcal{M}_{n}$.
Using the MS divergence gradient function, DIM employs a mini-batch gradient descent technique \cite{yu2017seqgan} to solve the optimization problem in Eq. \ref{eq:obj1} for the differentiable GAN-based imputation model.
To be more specific, similar as the studies \cite{bellemare2017cramer, salimans2018improving}, the discriminator is trained to maximize the MS divergence between the true underlying and generated data distributions, while the generator is trained by minimizing the MS divergence metric evaluated by the newly updated discriminator.
As a result, the DIM module employs the optimized GAN-based imputation model $\mathcal{M}_{n}$ to impute missing values in $\mathbf{X}_n$ via using Eq. \ref{eq:imputation}.


%% file: 5.estimation.tex
\section{Sample Size Estimation}
\label{sec:SSE}


With the support of the differentiable GAN-based imputation model in the DIM module, the goal of the \emph{sample size estimation} (SSE) module is to analytically infer the minimum sample size $n_\star$ (with $n_0\le n_\star\le N$) under accuracy guarantees.
With the SSE module, the imputation difference $\mathcal{D}(\mathcal{M}_{\star}, \mathcal{M}_{N})$ between the model $\mathcal{M}_{\star}$ trained on the size-$n_\star$ (in)complete samples from $\mathbf{X}$ and the model $\mathcal{M}_{N}$ trained on the given incomplete dataset $\mathbf{X}$ is well-controlled, i.e.,
\begin{equation}
\begin{array}{ll}
&  \mathbb{P}(\mathcal{D}(\mathcal{M}_{\star}, \mathcal{M}_{N})\le \varepsilon) \geq 1-\alpha,\\
\text{where} & \mathcal{D}(\mathcal{M}_{\star}, \mathcal{M}_{N})=\left(\mathbb{E}_{\mathbf{x}}\left[\mathbf{m}\odot(\bar{\mathbf{x}}_{\star}-\bar{\mathbf{x}}_{N})\right]^{2}\right)^{\frac{1}{2}}, %
\label{eq:target}
\end{array}
\end{equation}
where $\varepsilon$ is a user-tolerated error bound; $\alpha$ is a confidence level; $\mathbf{m}$ is the mask vector of $\mathbf{x}$; $\bar{\mathbf{x}}_{\star}$ and $\bar{\mathbf{x}}_{N}$ are the vectors reconstructed by $\mathcal{M}_{\star}$ and $\mathcal{M}_{N}$ over $\mathbf{x}$, respectively.

The core idea of SSE is that, the differentiablity of MS divergence enables \textsf{SCIS} to
determine the distributions of the parameters $\theta_\star$ and $\theta_N$ for GAN-based imputation models (i.e., $\mathcal{M}_{\star}$ and $\mathcal{M}_{N}$),
based on which we exploit an empirical estimation to calculate the probability $\mathbb{P}(\mathcal{D}(\theta_\star, \theta_N)\le \varepsilon)$, i.e., $\mathbb{P}(\mathcal{D}(\mathcal{M}_{\star}, \mathcal{M}_{N})\le \varepsilon)$.
Specifically, in the SSE module, we first estimate the distributions of the parameters $\theta_\star$ and $\theta_N$ with the given parameter $\theta_{0}$ of the initial model $\mathcal{M}_{0}$ by using the minimum distance estimator on differentiable MS divergence, as shown in Theorem \ref{theorem:1}.
Then, we employ an empirical estimation to calculate the probability $\mathbb{P}(\mathcal{D}(\theta_\star, \theta_N)\le \varepsilon)$, as presented in Proposition \ref{prop:2}.
Finally, with the estimation of $\mathbb{P}(\mathcal{D}(\theta_\star, \theta_N)\le \varepsilon)$, we use binary search to find the minimum size $n_\star$ (in)complete samples to satisfy the users' expectation on imputation accuracy in Eq. \ref{eq:target}.


To begin with, we study the estimation on the distribution of the parameter $\theta_{n}$ in an unknown differentiable GAN-based imputation model $\mathcal{M}_n$ with a size-$n$  ($n_0\le n\le N$) (in)complete sample set from $\mathbf{X}$, by using the parameter $\theta_{0}$ of the initial model $\mathcal{M}_{0}$.
Inspired by the asymptotic normality of minimum distance estimator \cite{newey1994large}, we estimate the parameter distribution of $\theta_{n}$ in $\mathcal{M}_n$ with the given $\theta_{0}$ of $\mathcal{M}_{0}$, as stated in Theorem \ref{theorem:1}.
For brevity, we make a notation that for two positive-value functions $g$ and $h$ over the same domain and upper bounded by a positive constant, $g \asymp h$ iff $g/h$ and $h/g$ are bounded by a constant.
In particular, the MS divergence gradient function makes Theorem \ref{theorem:1} work for GAN-based imputation models, while the corresponding theorem in \cite{park2019blinkml} is only targeted for traditional machine learning models.

\begin{theorem}\label{theorem:1}
Given the parameter $\theta_0$ of a differentiable generative adversarial imputation model $\mathcal{M}_0$, the conditional probability distribution of $\hat{\theta}_n$ with respect to $\mathcal{M}_n$ satisfies
\begin{equation}
\begin{array}{ll}
\label{eq:conditional}
&\hat{\theta}_n|\theta_{0} \to \mathcal{N}(\theta_{0}, \eta \mathbf{H}^{-1}) \\
\text { with } & \eta \asymp e^{\frac{6}{\lambda}}(1 + \frac{1}{\lambda^{\lfloor d/2 \rfloor}})^{2} \cdot \left(\frac{1}{n_0} - \frac{1}{n}\right),
\end{array}
\end{equation}
as $n_0 \to \infty$ and $n \to \infty$,
where $\lambda$ is a hyper-parameter in the MS divergence;
$\mathbf{H}$ is the invertible Hessian matrix of the MS divergence imputation loss function with the given parameter $\theta_{0}$; $\mathcal{N}(\theta_{0}, \eta \mathbf{H}^{-1})$ denotes a multivariate normal distribution with mean $\theta_{0}$ and covariance matrix $\eta \mathbf{H}^{-1}$.
\end{theorem}

\begin{proof}
To estimate the conditional probability distribution of $\hat{\theta}_n$, we first derive the distribution of $\hat{\theta}_{0} - \theta_{\infty}$ by the multivariate central limit theorem, where $\theta_{\infty}$ is the conceptual optimal parameter when the sample size approaches infinity.
Then, we infer the distribution of $\hat{\theta}_{0} - \hat{\theta}_{n}$ based on that of $\hat{\theta}_{0} - \theta_{\infty}$.
Finally, we exploit Bayes' theorem to estimate the distribution of $\hat{\theta}_{n}|\theta_{0}$ from the distribution of $\hat{\theta}_{0} - \hat{\theta}_{n}$.

Specifically, since the MS divergence is differentiable, ${\theta}_{0}$ is obtained by finding the parameter at which $g(\theta)$ becomes 0, i.e., ${\theta}_{0}$ satisfies $g(\theta_{0}) = 0$.
According to the mean-value theorem, there exists $\bar{\theta}_{0}$ between $\theta_{0}$ and $\theta_{\infty}$ that satisfies
\begin{equation} \nonumber
\begin{aligned}
    g'(\bar{\theta}_{0})\cdot(\theta_{0} - \theta_{\infty}) &= g(\theta_{0}) - g(\theta_{\infty}) = -g(\theta_{\infty}).
\end{aligned}
\end{equation}
Besides, we note that $\theta_{0}$ is simply an instance from the distribution of $\hat{\theta}_{0}$.
Therefore, we can find
\begin{align} \nonumber
    \hat{\theta}_{0} - \theta_{\infty} &= - g'(\bar{\theta}_{0})^{-1}g(\theta_{\infty}) \\ \nonumber
     &= -  g'(\bar{\theta}_{0})^{-1} \bigg(\sum_{j=1}^{n_0}\bigg[\sum_{i=1}^{n_0}\mathbf{P}^{\star}_{ij} (\bar{\mathbf{\mathbf{x}}}_{i} \odot \mathbf{m}_{i} \\\label{eq:second}
     & \qquad - {\mathbf{x}}_{j} \odot \mathbf{m}_{j})  \mathcal{T}(\mathbf{m_{i}})  \nabla_{\theta_{\infty}} \bar{\mathbf{x}}_{i} \bigg]\bigg)\\ \label{eq:clt}
    & \xrightarrow{n_0 \to \infty} \mathcal{N}\bigg(0, \frac{1}{n_{0}}\zeta(\lambda )\mathbf{H}^{-1}\mathbf{J}\mathbf{H}^{-1}\bigg) \\
     &=  \mathcal{N}\bigg(0, \frac{1}{n_{0}}\zeta(\lambda)\mathbf{H}^{-1}\bigg),\nonumber
\end{align}
where $\zeta(\lambda) \asymp {e^{\frac{2\kappa}{\lambda}}}(1+\frac{1}{\lambda^{\lfloor d/2 \rfloor}})^{2}$; $\mathbf{J}$ and $\mathbf{H}$ are the information matrix and the invertible Hessian matrix of the MS divergence imputation loss function with the given parameter $\theta_{0}$, respectively.
The transition from Eq. \ref{eq:second} to Eq. \ref{eq:clt} is based on the multidimensional central limit theorem and the fact that the simple complexity for Sinkhorn divergence is $O(\frac{e^{\frac{\kappa}{\lambda}}}{\sqrt{n}}(1+\frac{1}{\lambda^{\lfloor d/2 \rfloor}}))$ \cite{genevay2019sample}, where $\kappa = 2L|\mathcal{X}|+ \vert|f_c\vert|_{\infty}$. Here, $|\mathcal{X}|$ is the diameter $\rm{sup}\{\lVert \mathbf{x} - \mathbf{x'}\rVert \mid \mathbf{x}, \mathbf{x'} \in \mathcal{X}\}$
of $\mathcal{X}$. $L$ is the Lipschitz constant for the cost function $f_c$.
In our case, both $|\mathcal{X}|$ and $L$ will be 1, since the input data will be normalized to $[0,1]^{d}$ and the cost function $f_c$ is a squared two-norm function. Thus, $\zeta(\lambda) \asymp {e^{\frac{6}{\lambda}}}(1+\frac{1}{\lambda^{\lfloor d/2 \rfloor}})^{2}$.
Moreover, the information matrix equality $\mathbf{J} = - \mathbf{H}$ implies the last equation.

Then, with the estimated distribution of $\hat{\theta}_{0} - \theta_{\infty}$, we further infer the distribution of $\hat{\theta}_{0} - \hat{\theta}_{n}$.
To be more specific, since $\mathbf{X}_{n}$ can be regarded as a union of $\mathbf{X}_{0}$ and $\mathbf{X}_{n} - \mathbf{X}_{0}$, we introduce two random variables $V_1$ and $V_2$ independently following $\mathcal{N}(0,  \zeta(\lambda)\mathbf{H}^{-1})$ to capture the randomness of $\mathbf{X}_{0}$ and $\mathbf{X}_{n} - \mathbf{X}_{0}$, respectively.
Note that $\hat{\theta}_{0} - \theta_{\infty}$ $\rightarrow$ $\frac{1}{\sqrt{n_0}}V_{1}$.
In this way, we can find
\begin{align}\nonumber
    \hat{\theta}_{0} - \hat{\theta}_{n} =& (\hat{\theta}_{0} - \theta_{\infty}) - (\hat{\theta}_{n} - \theta_{\infty})\\\nonumber
     =& -\frac{1}{\sqrt{n_0}}\mathbf{H}^{-1}\Bigg(\frac{1}{\sqrt{n_0}} \sum_{j=1}^{n_0} \bigg[\sum_{i=1}^{n_0}\mathbf{P}^{\star}_{ij}(\bar{\mathbf{\mathbf{x}}}_{i} \odot \mathbf{m}_{i} \\\nonumber
    &- {\mathbf{x}}_{j} \odot \mathbf{m}_{j})  \mathcal{T}(\mathbf{m}_{i})  \nabla_{\theta_{\infty}} \bar{\mathbf{x}}_{i} \bigg]\Bigg)\\ \nonumber
    & + \frac{1}{\sqrt{n}}\mathbf{H}^{-1}\Bigg(\frac{\sqrt{n_0}}{\sqrt{n}}\frac{1}{\sqrt{n_0}}\sum_{j=1}^{n_0} \Bigg[\sum_{i=1}^{n}\mathbf{P}^{\star}_{ij}(\bar{\mathbf{\mathbf{x}}}_{i} \odot \mathbf{m}_{i} \\\nonumber
    &- {\mathbf{x}}_{j} \odot \mathbf{m}_{j})  \mathcal{T}(\mathbf{m}_{i})  \nabla_{\theta_{\infty}} \bar{\mathbf{x}}_{i} \Bigg]\\ \nonumber
    & + \frac{\sqrt{n-n_0}}{\sqrt{n}} \frac{1}{\sqrt{n - n_0}} \sum_{j=n_0+1}^{n} \Bigg[\sum_{i=1}^{n}\mathbf{P}^{\star}_{ij}(\bar{\mathbf{\mathbf{x}}}_{i} \odot \mathbf{m}_{i} \\\nonumber
    &- {\mathbf{x}}_{j} \odot \mathbf{m}_{j})  \mathcal{T}(\mathbf{m}_{i})  \nabla_{\theta_{\infty}} \bar{\mathbf{x}}_{i} \Bigg]\Bigg).
\end{align}
Moreover, we exploit the fact that $\hat{\theta}_{0} - \hat{\theta}_{n}$, which is the sum of two independent normal random variables, asymptotically follows a normal distribution. Thus, when both $n_0$ and $n$ approach infinity, $\hat{\theta}_{0} - \hat{\theta}_{n}$ follows
\begin{equation}\nonumber
\begin{aligned}
    \hat{\theta}_{0} - \hat{\theta}_{n} &\xrightarrow{n_0 \text{ and }  n \to \infty} \left(\frac{1}{\sqrt{n_0}} - \frac{\sqrt{n_0}}{n}\right)V_{1} - \frac{\sqrt{n - n_0}}{n}V_{2}\\
     &\sim \mathcal{N}\left(0, \left(\frac{1}{n_0} - \frac{1}{n}\right) \cdot \zeta(\lambda) \cdot \mathbf{H}^{-1}\right). \label{eq:comb1}
\end{aligned}
\end{equation}

Finally, we exploit Bayes' theorem to estimate the distribution of $\hat{\theta}_{n}|\theta_{0}$ based on the distribution of $\hat{\theta}_{0} - \hat{\theta}_{n}$.
Observe that $\hat{\theta}_{0} - \hat{\theta}_{n}$ and $\hat{\theta}_{n} - \theta_{\infty}$ are independent because the covariance between them is zero, i.e.,
\begin{equation}
    \begin{aligned}\nonumber
    Cov(\hat{\theta}_{0} - \hat{\theta}_{n}&, \hat{\theta}_{n} - \theta_{\infty}) = \frac{1}{2}\Big(Var(\hat{\theta}_{0} - \hat{\theta}_{n} + \hat{\theta}_{n} - \theta_{\infty}) \\
    &- Var(\hat{\theta}_{0} - \hat{\theta}_{n}) - Var(\hat{\theta}_{n} - \theta_{\infty})\Big)\\
    =& \frac{1}{2} \left(\frac{1}{n_0} - \left(\frac{1}{n_0} - \frac{1}{n}\right) - \frac{1}{n}\right)\cdot \zeta(\lambda) \cdot\mathbf{H}^{-1} \\=& 0.
    \end{aligned}
\end{equation}
Thus,
\begin{equation}\nonumber
Var(\hat{\theta}_{0} - {\theta}_{n}) =  Var(\hat{\theta}_{0} - \hat{\theta}_{n}|\theta_{n}) = \eta \mathbf{H}^{-1},
\end{equation}
which implies
\begin{equation}\nonumber
\begin{aligned}
\label{eq:condistri}
    &\hat{\theta}_{0} \sim (\theta_{n}, \eta \mathbf{H}^{-1}),\\
    \text {with}\ & \eta \asymp e^{\frac{6}{\lambda}}(1 + \frac{1}{\lambda^{\lfloor d/2 \rfloor}})^{2} \cdot \left(\frac{1}{n_0} - \frac{1}{n}\right).
    \end{aligned}
\end{equation}
By using Bayes' theorem,
\begin{equation} \nonumber
    \mathbb{P}(\theta_{n}|\theta_{0}) = \frac{1}{Z}\mathbb{P}(\theta_{0}|\theta_{n})\mathbb{P}(\theta_{n}),
\end{equation}
for some normalization constant $Z$. Since we don't have any extra information about the prior probability $\mathbb{P}(\theta_{n})$, we set a constant to $\mathbb{P}(\theta_{n})$.
Therefore, we can find that, the conditional probability distribution of $\hat{\theta}_n$ with respect to $\mathcal{M}_n$ satisfies
\begin{equation} \nonumber
\begin{aligned}
    &\hat{\theta}_n|\theta_{0} \to \mathcal{N}(\theta_{0}, \eta \mathbf{H}^{-1}), \\ \text{ with } & \eta \asymp e^{\frac{6}{\lambda}}\left(1 + \frac{1}{\lambda^{\lfloor d/2 \rfloor}}\right)^{2} \cdot \left(\frac{1}{n_0} - \frac{1}{n}\right).
    \end{aligned}
\end{equation}
\end{proof}

For the differentiable GAN-based imputation model $\mathcal{M}_0$, the invertible Hessian matrix $\mathbf{H}$ in Eq. \ref{eq:conditional} can be computed by
\begin{equation}\nonumber
\begin{aligned}
\mathbf{H} =& \frac{1}{n_0}\sum_{j=1}^{n_0}\sum_{i=1}^{n_0}\mathbf{P}^{\star}_{ij}  \Bigg((\sum_{k=1}^{d}{m}^{2}_{ik}(m_{ik} \bar{x}_{ik} - {m}_{jk} {x}_{jk}) \nabla^2_{\theta} \bar{\mathbf{x}}_{ik})\\
&+ [\mathcal{T}(\mathbf{m}_{i})  \nabla_{\theta} \bar{\mathbf{x}}_{i} ]^{\top}  \mathcal{T}(\mathbf{m}_{i})  \nabla_{\theta} \bar{\mathbf{x}}_{i} \Bigg) \\
\approx&  \frac{1}{n_0}\sum_{j=1}^{n_0}\sum_{i=1}^{n_0}\mathbf{P}^{\star}_{ij}  [\mathcal{T}(\mathbf{m}_{i})  \nabla_{\theta} \bar{\mathbf{x}}_{i} ]^{\top} \mathcal{T}(\mathbf{m}_{i})  \nabla_{\theta} \bar{\mathbf{x}}_{i}.
\end{aligned}
\label{eq:hessian_initial}
\end{equation}
In particular, $\nabla^2_{\theta} \bar{\mathbf{x}}_{ik}$ indicates the Hessian matrix of $\bar{\mathbf{x}}_{ik}$ with respect to $\theta$.
The approximation follows by ignoring the first term in the first equation. The first term is negligible, since the reconstructed vector $\bar{\mathbf{x}}_i$ happens to be very close to the observed values in $\mathbf{x}_{i} \in \mathbf{X}_0$, where
the initial model $\mathcal{M}_0$ has been trained on $\mathbf{X}_0$ \cite{nilsen2019efficient}.

As a result,  with the given parameter $\theta_0$ and the estimated probability distribution of the parameter $\theta_n$, we employ an empirical estimation to determine $\mathbb{P}(\mathcal{D}(\theta_{0}, \theta_{n})\le \varepsilon)$.
It is inspired by the conservative estimation developed on Hoeffding's inequatity in \cite{park2019blinkml}.
We can guarantee the probability $\mathbb{P}(\mathcal{D}(\theta_{0}, \theta_{n})\le \varepsilon)$ at least  ($1 - \alpha$) by requesting the empirical estimation of $\mathbb{P}(\mathcal{D}(\theta_{0}, \theta_{n})\le \varepsilon)$  no less than a certain value, as described in Proposition \ref{prop:2}.
In essence, we exploit interval estimation to derive a new lower bound for the empirical statistic of $\mathbb{P}(\mathcal{D}(\theta_{0}, \theta_{n})\le \varepsilon)$ in Proposition \ref{prop:2}.
The lower bound is orthogonal to the one in \cite{park2019blinkml} by considering an additional hyper-parameter $\beta$ to bound the error between $\mathbb{P}(\mathcal{D}(\theta_{0}, \theta_{n})\le \varepsilon)$ and its empirical statistic.
\begin{proposition}\label{prop:2}
\label{lemma:probability}
Let $f(\theta_{n})$ be the density function of the conditional probability distribution  $\hat{\theta}_n|\theta_{0}$.  Suppose $\theta_{n, 1}, \cdots, \theta_{n, k}$ are i.i.d. samples drawn from $f(\theta_{n})$, where $k$ is the number of parameter sampling. The inequality $\mathbb{P}(\mathcal{D}(\theta_{0}, \theta_{n})\le \varepsilon) \geq 1-\alpha$ holds, if $\varepsilon$ satisfies that,
\begin{equation}\nonumber
\begin{aligned}
\mathbb{P}(\mathcal{D}(\theta_{0}, \theta_{n})\le \varepsilon)
&\approx \frac{1}{k}\sum_{i=1}^{k}\mathcal{I}[\mathcal{D}(\theta_{0}, \theta_{n, i}) \le \varepsilon] \\&\ge \frac{1 - \alpha}{1 - \beta} + \sqrt{\frac{\rm{log} (\beta)}{-2k}},
\label{eq:lemma1-app}
\end{aligned}
\end{equation}
where $\beta$ is a hyper-parameter ($0 <  \beta  \le \alpha\le 1$) and $\mathcal{I}$ is the indicator function that returns 1 if its argument is true, otherwise returns 0.
\end{proposition}

\begin{proof}
By applying Hoeffding's inequality,
\begin{align}\nonumber
&\mathbb{P}(a  \ge b -   \varepsilon_{1}) \ge  1 -  e^{-2k\varepsilon_{1}^{2}},\\ \nonumber
\rm{where} \quad & a = \int \mathcal{I}[\mathcal{D}(\theta_{0}, \theta_{n})\le \varepsilon]f(\theta_{n})d\theta_{n},\\
&b = \frac{1}{k}\sum_{i=1}^{k} \mathcal{I}[\mathcal{D}(\theta_{0}, \theta_{n, i}) \le \varepsilon]. \nonumber
\end{align}
Hence, for a given $k$, we can select $\varepsilon_{1} \ge \displaystyle{\sqrt{\frac{\rm{log}(\beta)}{-2k}}}$, so as to have ($1 - \beta$) probability of assuring $a$ not less than $b -  \varepsilon_{1}$.

Furthermore, to guarantee that $\mathbb{P}(\mathcal{D}(\theta_{0}, \theta_{n}) \le \varepsilon) \ge 1 - \alpha$, we will take conservative estimation, i.e., ensuring that
\begin{align}\nonumber
   \frac{1}{k}\sum_{i=1}^{k}\mathcal{I}[\mathcal{D}(\theta_{0}, \theta_{n,i}) \le \varepsilon] &\ge \frac{1 - \alpha}{1 - \beta} + \varepsilon_{1}\\\nonumber
   &\ge \frac{1 - \alpha}{1 - \beta} + \sqrt{\frac{\rm{log} (\beta)}{-2k}}.
\end{align}
Thus, we can finally have at least $(1 - \beta)\cdot \displaystyle{\frac{1 - \alpha}{1 - \beta}} = (1 - \alpha)$ confidence on guaranteeing $\mathcal{D}(\theta_{0}, \theta_{n}) \le \varepsilon$.
\end{proof}

Therefore, according to Proposition \ref{lemma:probability}, the probability $\mathbb{P}(\mathcal{D}(\theta_{n}, \theta_N) \le \varepsilon)$ can be approximated by
\begin{equation}\nonumber
\begin{aligned}
\mathbb{P}(\mathcal{D}(\theta_{n}, \theta_N) \le \varepsilon) \approx \frac{1}{k} \sum_{i=1}^{k} \mathcal{I}[\mathcal{D}(\theta_{{n}, i}, \theta_{N, i})\le \varepsilon],
\label{eq:joint}
\end{aligned}
\end{equation}
where $(\theta_{{n}, i}, \theta_{N, i})$ is a parameter pair sampled from the conditional probability distributions $\hat{\theta}_{n, i}|\theta_{0}$ and $\hat{\theta}_{N}|\theta_{n, i}$, respectively.
Moreover, $\mathbb{P}(\mathcal{D}(\theta_{n}, \theta_N) \le \varepsilon)$ is an increasing function with respect to $n$, which is derived in the similar way as \cite{park2019blinkml}.

Based on Proposition \ref{prop:2}, the SSE module uses binary search to find the minimum size $n_\star$ (with $n_0\le n_\star\le N$) of (in)complete samples in $\mathbf{X}$ to satisfy the users' expectation on imputation accuracy. The search procedure terminates, if $n_\star$ satisfies Eq. \ref{eq:target} while ($n_\star-1$) does not.
Namely,
\begin{equation}\nonumber
\begin{aligned}
\frac{1}{k} \sum_{i=1}^{k} \mathcal{I}[\mathcal{D}(\theta_{{\star}, i}, \theta_{N, i})\le \varepsilon]&\ge \frac{1 - \alpha}{1 - \beta} + \sqrt{\frac{\rm{log} (\beta)}{-2k}} \\
&> \frac{1}{k} \sum_{i=1}^{k} \mathcal{I}[\mathcal{D}(\theta_{{\star-1}, i}, \theta_{N, i})\le \varepsilon],
\label{eq:terminate}
\end{aligned}
\end{equation}
where $\theta_{\star-1}$ is the parameter of the GAN-based imputation model trained with the size-($n_\star-1$) (in)complete samples.


%% file: 6.experiment.tex
\section{Experiment}
\label{sec:experiment}
In this section, we evaluate the performance of our proposed system \textsf{SCIS} via comparisons with eleven state-of-the-art imputation methods.
All algorithms were implemented in Python.
The experiments were conducted on an Intel Core 2.80GHz server with TITAN Xp 12GiB (GPU) and 192GB RAM, running Ubuntu 18.04 system.

\textbf{Datasets.}
In the experiments, we use six public real-world incomplete COVID-19 datasets over several scenarios.
Table \ref{tab:dataset} lists the information of each dataset, including the number of samples ($\#$Samples), features ($\#$Features), and missing rate, respectively.
In particular, (i) COVID-19 trials tracker\footnote{\scriptsize{https://www.kaggle.com/frtgnn/covid19-trials-tracker?select=Covid-19+Trials Tracker.csv}}  (\emph{\textbf{Trial}}) dataset shows the clinical trial registries on studies of COVID-19 all around the world and tracks the availability of the studies' results. It contains 6,433 trials with 9 features, taking an about 9.63\% missing rate.
(ii) Emergency declarations\footnote{\scriptsize{https://github.com/GoogleCloudPlatform/covid-19-open-data/blob/main/docs/table-emergency-declarations.md}} (\emph{\textbf{Emergency}}) dataset, which is aggregated by the Policy Surveillance Program at the Temple University Center for Public Health Law Research, contains emergency declarations and mitigation policies for each US state starting on January 20, 2020.
It includes 8,364 samples with 22 features, taking an about 62.69\% missing rate.
(iii) Government response\footnote{\scriptsize{https://github.com/GoogleCloudPlatform/covid-19-open-data/blob/main/docs/table-government-response.md}} (\emph{\textbf{Response}}) dataset contains a summary of a government's response to the events, including a stringency index, collected from the university of Oxford \cite{hale2021global}.
It includes 200,737 samples with 19 dimensions and a 5.66\% missing rate.
(iv) Symptom search trends\footnote{\scriptsize{https://github.com/GoogleCloudPlatform/covid-19-open-data/blob/main/docs/table-search-trends.md}}  (\emph{\textbf{Search}}) dataset for COVID-19 shows how Google search patterns for different symptoms change based on the relative frequency of searches for each symptom in 2,792  specific regions. It totally contains 948,762 samples with 424 symptoms and an 81.35\% missing rate.
(v) Daily weather\footnote{\scriptsize{https://github.com/GoogleCloudPlatform/covid-19-open-data/blob/main/docs/table-weather.md}}
 (\emph{\textbf{Weather}}) dataset shows the 9 daily weather attributes from the nearest station reported by National Oceanic and Atmospheric Administration in specific regions. It includes 4,911,011 samples from 19,284 regions with a 21.56\% missing rate.
(vi) COVID-19 case surveillance public use\footnote{\scriptsize{https://data.cdc.gov/Case-Surveillance/COVID-19-Case-Surveillance-Public-Use-Data/vbim-akqf}}
(\emph{\textbf{Surveil}}) dataset shows the 7 clinical and symptom features for 22,507,139 cases shared by the centers for disease control and prevention, taking a 47.62\% missing rate.
In particular, the licenses for \emph{Trial}, \emph{Emergency}, \emph{Response}, \emph{Search}, \emph{Weather}, and \emph{Surveil} datasets are ODbL, CC BY, CC BY, CC0, CC BY, and CC0, respectively.

\begin{table}[t]
\small
\centering
  \caption{Dataset statistics}
\vspace*{-0.08in}
  \label{tab:dataset}
    \setlength{\tabcolsep}{10pt}
  \begin{tabular}{|c|c|c|c|}
	\hline
    \textbf{Name}& \textbf{\#Samples}& \textbf{\#Features}& \textbf{Missing rate}\\\hline
    \emph{Trial} &6,433  &9  &9.63\% \\\hline
    \emph{Emergency} &8,364  &22  &62.69\% \\\hline
    \emph{Response} &200,737&19&5.66\%\\\hline
    \emph{Search} &948,762&424&81.35\%\\\hline
    \emph{Weather} &4,911,011&9&21.56\%\\\hline
    \emph{Surveil} &22,507,139&7&47.62\%\\\hline

\end{tabular}
\end{table}

\begin{table*}[t]\small
\centering
\setlength{\tabcolsep}{6pt}
\caption{Performance comparison of imputation methods over Trial, Emergency, and Response}
\vspace*{-0.08in}
\label{Tab:all_imputation1}
\begin{tabular}{cccccccccc}
\hline
\multirow{2}{*}{Method} &\multicolumn{3}{c}{\emph{Trial}}&\multicolumn{3}{c}{\emph{Emergency}}&\multicolumn{3}{c}{\emph{Response}}\\\cline{2-10}
 & RMSE (Bias)& Time (s) &$R_t$ (\%)& RMSE (Bias)& Time (s) &$R_t$ (\%)& RMSE (Bias)& Time (s) &$R_t$ (\%)\\ \hline
\multicolumn{1}{c}{MissF}&0.417 ($\pm$ 0.023)&152&100&$-$&$-$&$-$&$-$&$-$&$-$\\
\multicolumn{1}{c}{Baran}& 0.412($\pm$ 0.025)&1,204&100&$-$&$-$&$-$&$-$&$-$&$-$\\
\multicolumn{1}{c}{MICE}& 0.402($\pm$ 0.021)&592&100&$-$&$-$&$-$&$-$&$-$&$-$\\
\multicolumn{1}{c}{DataWig}&0.458 ($\pm$ 0.042)&329&100&0.389 ($\pm$ 0.031)&3,252&100&$-$&$-$&$-$\\
\multicolumn{1}{c}{RRSI}&0.398 ($\pm$ 0.011)&413&100&0.358 ($\pm$ 0.025)&4,673&100&$-$&$-$&$-$\\
\multicolumn{1}{c}{MIDAE}&0.445 ($\pm$ 0.025)&532&100&0.378 ($\pm$ 0.032)&2,321&100&$-$&$-$&$-$\\
\multicolumn{1}{c}{VAEI}&0.463 ($\pm$ 0.023)&321&100&0.398 ($\pm$ 0.036)&892&100&$-$&$-$&$-$\\
\multicolumn{1}{c}{MIWAE}&0.396 ($\pm$ 0.015)&892&100&$-$&$-$&$-$&$-$&$-$&$-$\\
\multicolumn{1}{c}{EDDI}&0.404 ($\pm$ 0.018)&132&100&0.363 ($\pm$ 0.019)&662&100&$-$&$-$&$-$\\
\multicolumn{1}{c}{HIVAE}&0.396 ($\pm$ 0.020)&124&100&0.357 ($\pm$ 0.0.20)&450&100&0.398 ($\pm$ 0.018)&969&100\\\hline
GINN&0.432 ($\pm$ 0.030)&327&100&0.373 ($\pm$ 0.028)&1,572&100&$-$&$-$&$-$\\
\textsf{SCIS}-GINN &0.430 ($\pm$ 0.028)&311&\textbf{16.72}&0.369 ($\pm$ 0.025)&962&13.92&0.412 ($\pm$ 0.023)&7,263&2.53\\\hline
GAIN &0.398 ($\pm$ 0.024)&90&100&0.352 ($\pm$ 0.025)&340&100&0.396 ($\pm$ 0.031)&649&100\\
\textsf{SCIS}-GAIN &\textbf{0.386 ($\pm$ 0.017)}&\textbf{85}&23.56&\textbf{0.342 ($\pm$ 0.019)}&\textbf{286}&\textbf{12.32}&\textbf{0.389 ($\pm$ 0.021)}&\textbf{595}&\textbf{1.50}\\\hline
\end{tabular}
\vspace*{0.05in}
\end{table*}

\begin{table*}[t]\small
\centering
\setlength{\tabcolsep}{6pt}
\caption{Performance comparison of imputation methods over Search, Weather, and Surveil}
\vspace*{-0.08in}
\label{Tab:all_imputation2}
\begin{tabular}{ccccccccccccccccccc}
\hline
\multirow{2}{*}{Method} &\multicolumn{3}{c}{\emph{Search}}&\multicolumn{3}{c}{\emph{Weather}} &\multicolumn{3}{c}{\emph{Surveil}}\\\cline{2-10}
 & RMSE (Bias)& Time (s) &$R_t$ (\%)& RMSE (Bias)& Time (s) &$R_t$ (\%)& RMSE (Bias)& Time (s) &$R_t$ (\%)\\ \hline
\multicolumn{1}{c}{HIVAE}&$-$&$-$&$-$&0.174 ($\pm$ 0.014)&14,612&100&0.440 ($\pm$ 0.008)&35,041&100\\\hline
GINN&$-$&$-$&$-$&$-$&$-$&$-$&$-$&$-$&$-$\\
\textsf{SCIS}-GINN &$-$&$-$&$-$&0.243 ($\pm$ 0.031)&78,621 & 1.92 &$-$&$-$ &$-$\\\hline
GAIN & 0.252 ($\pm$ 0.014)&78,121&100 &0.165 ($\pm$ 0.014)&9,252&100 &0.440 ($\pm$ 0.012)&29,102&100\\
\textsf{SCIS}-GAIN &\textbf{0.250 ($\pm$ 0.013)}& \textbf{3,534}&\textbf{0.78}&\textbf{0.163 ($\pm$ 0.016)}&\textbf{2,275}&\textbf{1.90}&\textbf{0.438 ($\pm$ 0.013)}&\textbf{3,019}&\textbf{0.67}\\\hline
\end{tabular}
\end{table*}

\begin{figure*}[t]
\centering
\raisebox{-1cm}{\includegraphics[width=0.88\linewidth]{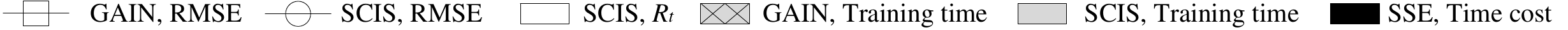}}
\vspace*{-0.02in}\\
\hspace*{-0.1in}
\subfigure[\emph{Trial}]{
\raisebox{-0.2cm}{\includegraphics[width=0.32\linewidth]{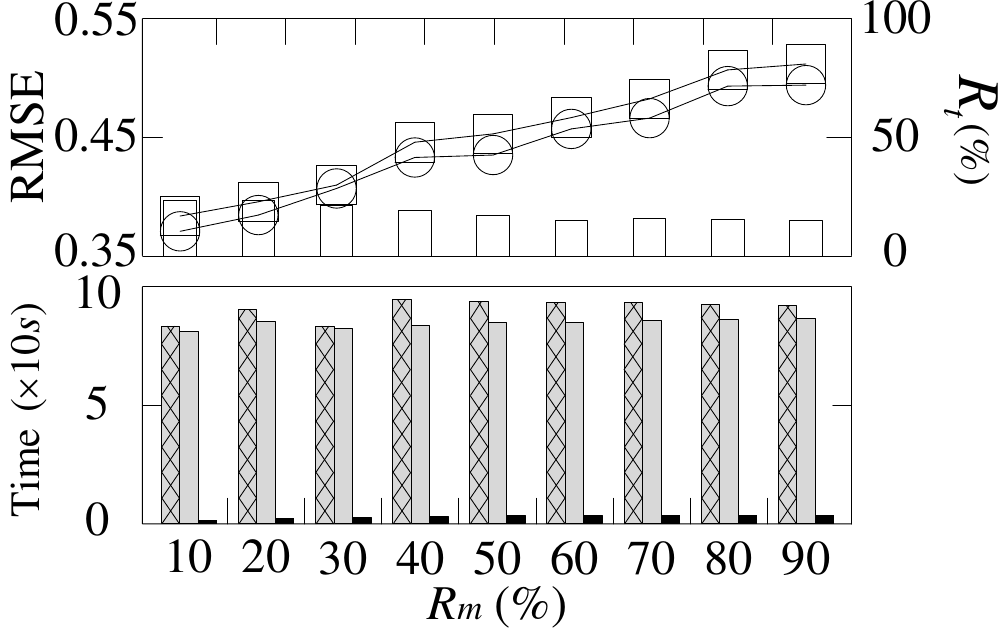}}}
\subfigure[\emph{Emergency}]{
\raisebox{-0.2cm}{\includegraphics[width=0.32\linewidth]{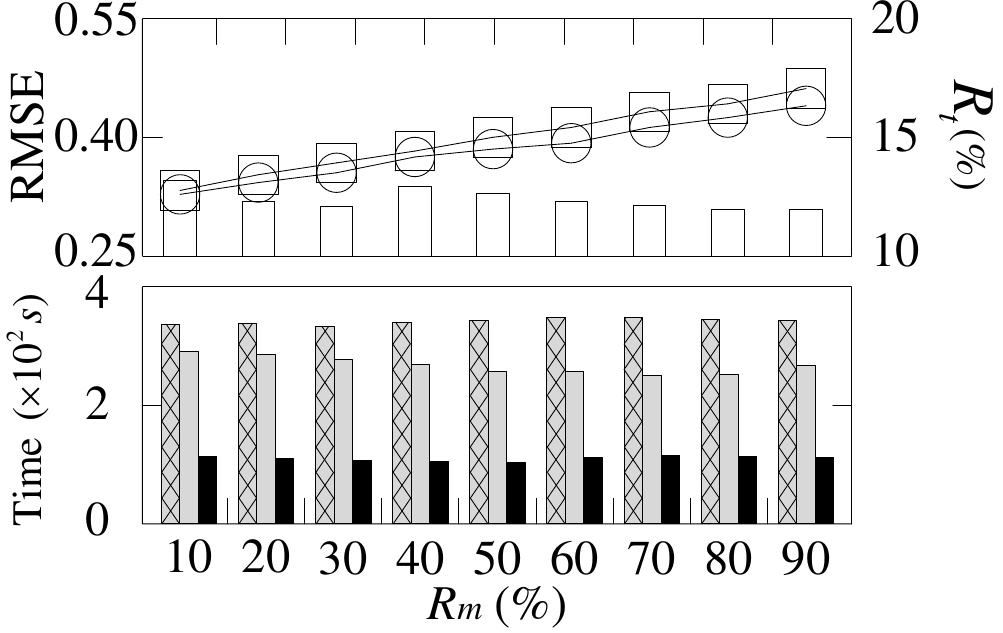}}}
\subfigure[\emph{Response}]{
\raisebox{-0.2cm}{\includegraphics[width=0.32\linewidth]{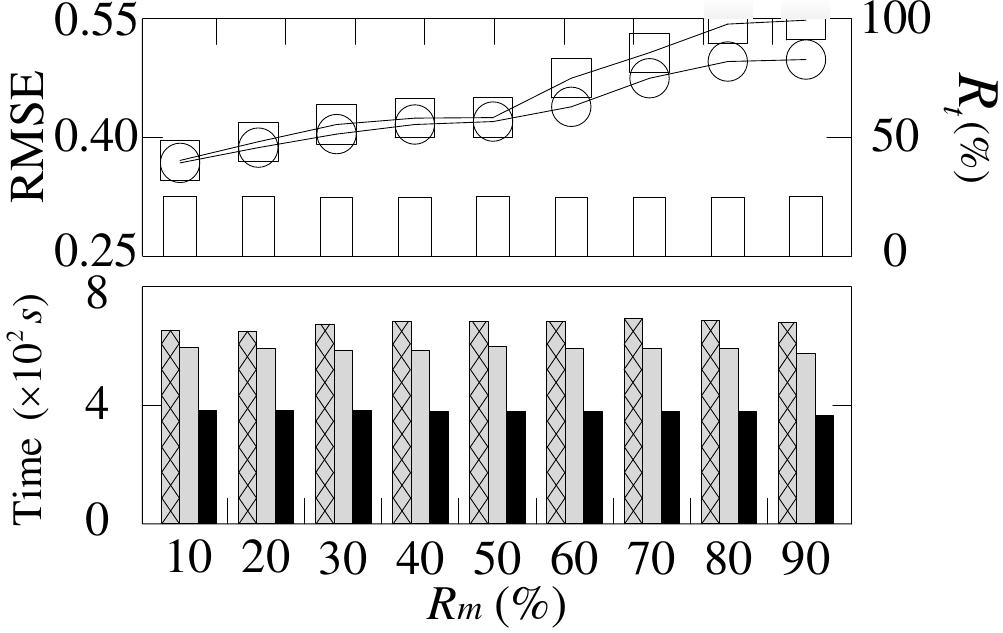}}}
\vspace*{-0.15in}
\\\hspace*{-0.1in}
\subfigure[\emph{Search}]{
\raisebox{-0.2cm}{\includegraphics[width=0.315\linewidth]{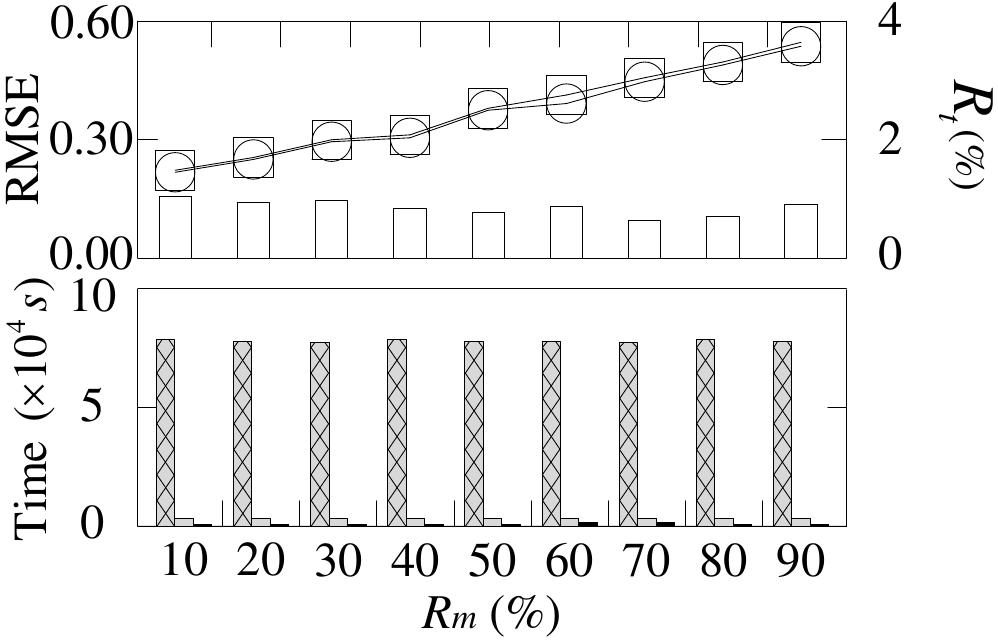}}}
\subfigure[\emph{Weather}]{
\raisebox{-0.2cm}{\includegraphics[width=0.32\linewidth]{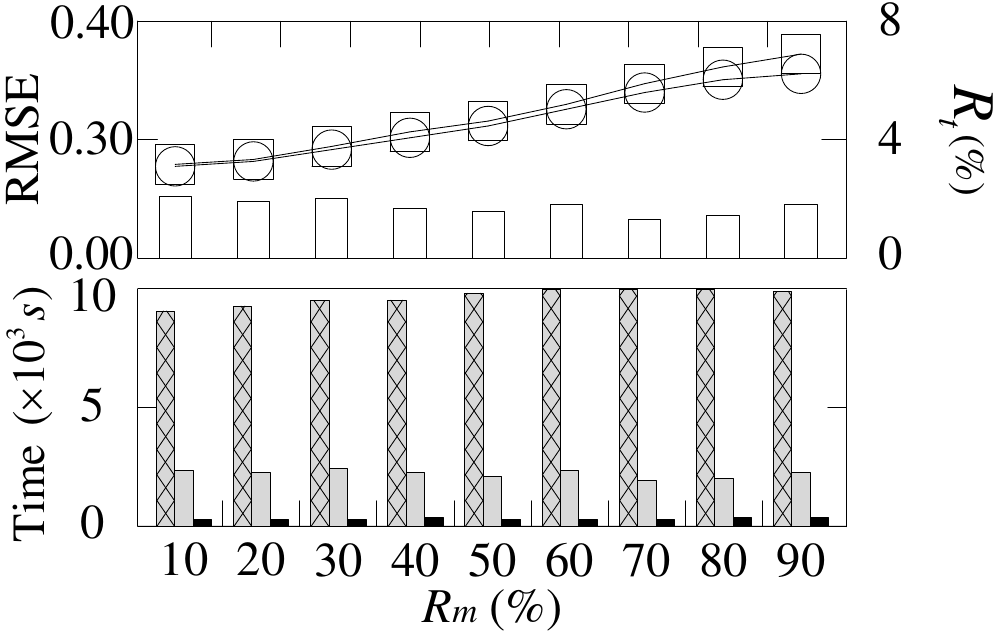}}}
\subfigure[\emph{Surveil}]{
\raisebox{-0.2cm}{\includegraphics[width=0.32\linewidth]{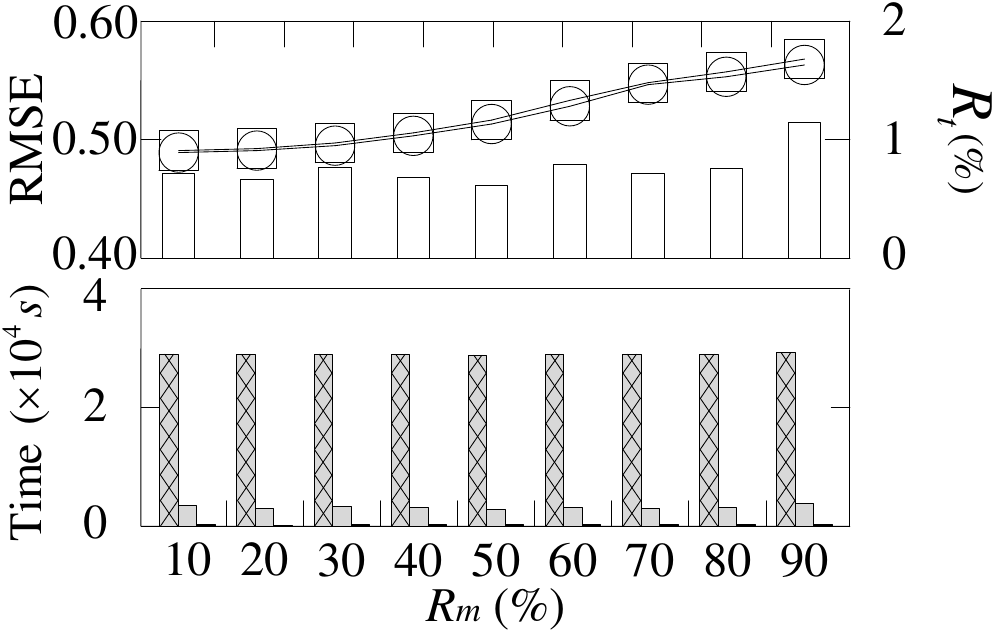}}}\\
\vspace*{-0.1in}
\caption{The performance of \textsf{SCIS} vs. $R_m$}
\vspace*{-0.25in}
\label{fig:MR}
\end{figure*}

\begin{figure*}[t]
\centering
\raisebox{-1cm}{\includegraphics[width=0.75\linewidth]{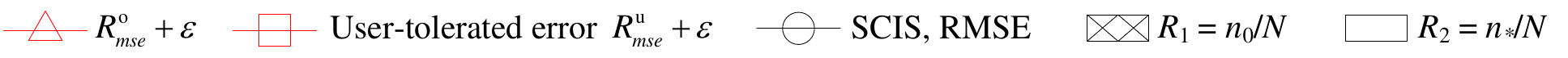}}
\vspace*{-0.02in}
\\ \vspace*{-0.03in}
\hspace*{-0.4in}
\subfigure[\emph{Trial}]{
\raisebox{-0.2cm}{\includegraphics[width=0.275\linewidth]{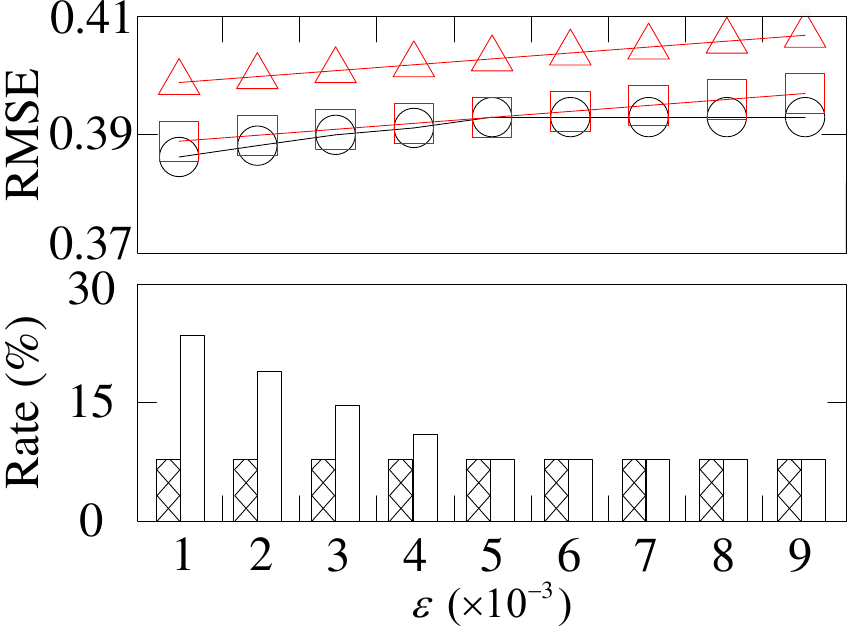}}}\hspace*{0.33in}
\subfigure[\emph{Emergency}]{
\raisebox{-0.2cm}{\includegraphics[width=0.275\linewidth]{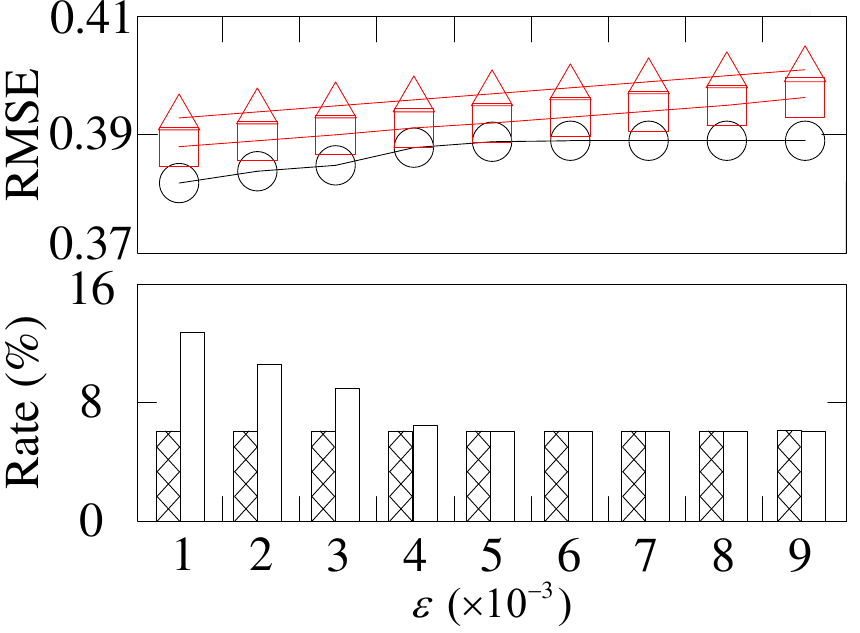}}}\hspace*{0.33in}
\subfigure[\emph{Response}]{
\raisebox{-0.2cm}{\includegraphics[width=0.275\linewidth]{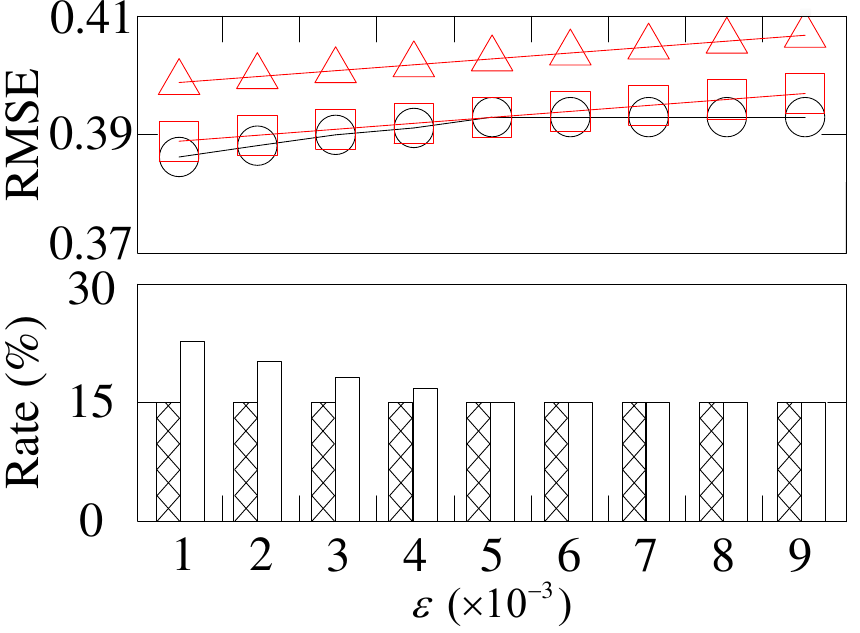}}}
\vspace*{-0.1in}
\\\hspace*{-0.45in}
\subfigure[\emph{Search}]{
\raisebox{-0.2cm}{\includegraphics[width=0.275\linewidth]{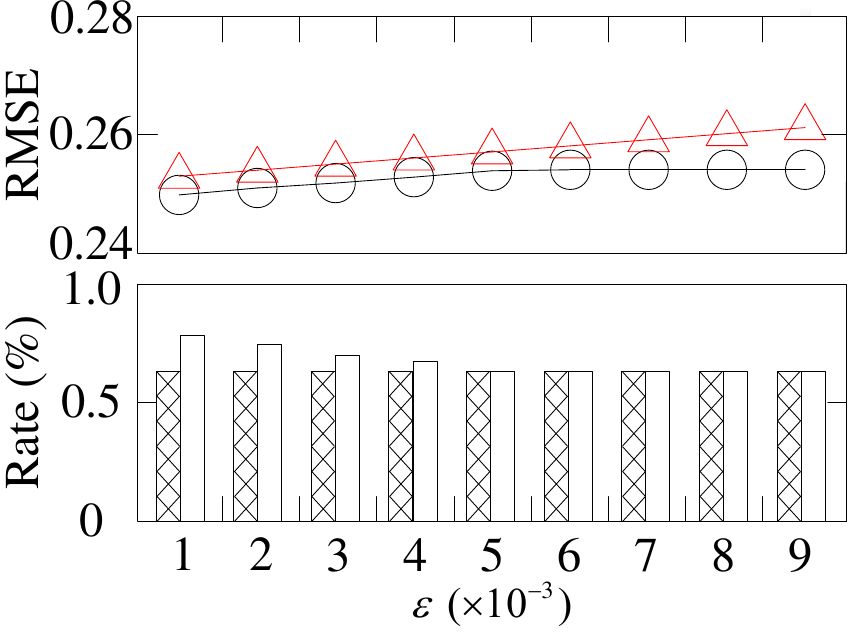}}}\hspace*{0.33in}
\subfigure[\emph{Weather}]{
\raisebox{-0.2cm}{\includegraphics[width=0.275\linewidth]{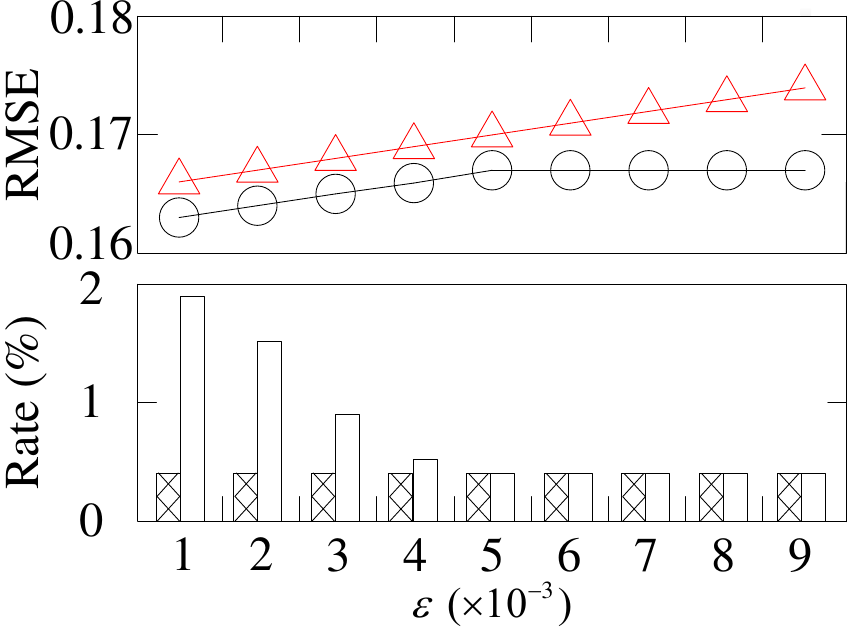}}}\hspace*{0.33in}
\subfigure[\emph{Surveil}]{
\raisebox{-0.2cm}{\includegraphics[width=0.275\linewidth]{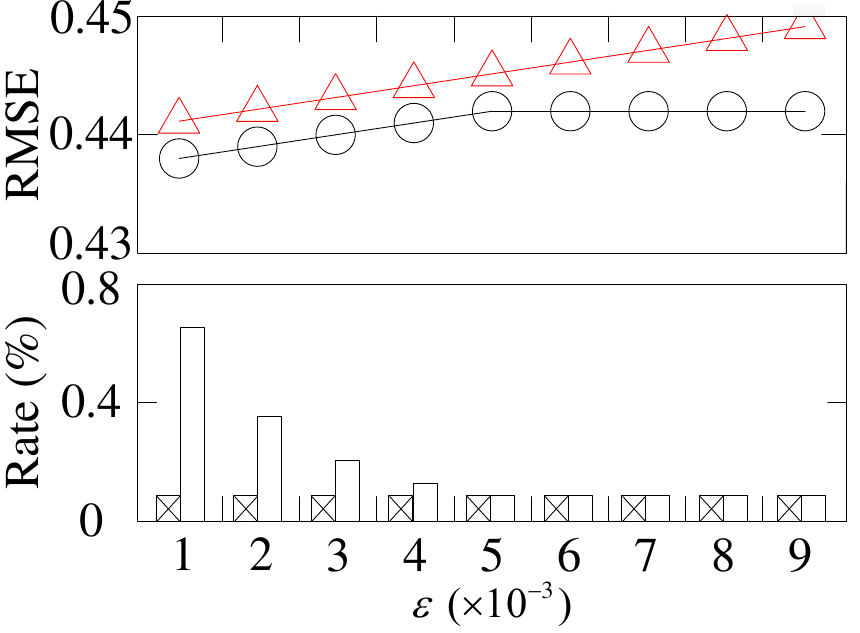}}}\\
\vspace*{-0.1in}
\caption{The performance of \textsf{SCIS} vs. $\varepsilon$}
\label{fig:Varepsilon}
\vspace*{-0.1in}
\end{figure*}

\textbf{Metrics.}
In the evaluation, we use the \emph{training time} and \emph{root mean squared error} (RMSE)  to measure the efficiency and effectiveness of imputation models.
We also report the \emph{training sample rate} $R_t$, i.e., how many samples are used for training models (100\% for basic original ones and $\frac{n_\star}{N}\times 100\%$ for \textsf{SCIS}).
The smaller the metric value, the better the imputation performance.
To obtain the RMSE values, we randomly remove 20\% observed values during training for imputation, and thus we use these observed values as the ground-truth to the missing values. In evaluation, each value is reported by averaging five times of experimental results under different data random divisions.

\textbf{Imputation methods.}
In the experiments, the baselines include eleven state-of-the-art imputation methods, namely three machine learning ones: MissF, Baran, and MICE, two MLP-based ones: DataWig and RRSI, four AE-based ones: MIDAE, VAEI, EDDI, and HIVAE, and two GAN-based ones: GINN and GAIN.
We evaluate \textsf{SCIS} on top of the GAN-based imputation methods.

\textbf{Implementation details.}
For all imputation methods, we thank the authors of each algorithm for providing the source codes, so that we directly utilize these source codes with default parameter settings in our experiment evaluation.
Specifically, for all machine learning imputation methods, the learning rate is 0.3, and the number of iterations is 100.
The number of decision trees in MissFI is 100.
Baran employs AdaBoost as the prediction model.
The imputation times in MICE are 20.
For all deep learning imputation methods, the learning rate is 0.001, the dropout rate is 0.5, the training epoch is 100, and the batch size is 128.
The ADAM algorithm is utilized to train networks.
MIDAE is a 2-layer with 128 units per layer network.
For VAEI, the encoder and decoder are fully connected networks with two hidden layers, each with 20 neurons per layer, and the latent space is 10-dimensional.
HIVAE uses only one dense layer for all the parameters of the encoder and decoder, each with 10 neurons per layer.
In GINN, the discriminator is a simple 3-layer feed-forward network trained 5 times for each optimization step of the generator.
In GAIN, both generator and discriminator are modeled as 2-layer fully connected network.
Moreover, in \textsf{SCIS}, the hyper-parameter $\lambda$ is 130, the confidence level $\alpha$ is 0.05, the hyper-parameter $\beta$ is 0.01, the number of parameter sampling $k$ in SSE is 20, the user-tolerated error bound $\varepsilon$ is 0.001.
The initial sample size $n_0$ is 500 for \emph{Trial}, 500 for \emph{Emergency}, 2,000 for \emph{Response}, 6,000 for \emph{Search}, 20,000 for \emph{Weather}, and 20,000 for \emph{Surveil}.
The validation sample size $N_v$ is equal to $n_0$.

\subsection{Scalability Evaluation}

Table~\ref{Tab:all_imputation1} and Table~\ref{Tab:all_imputation2} report the performance of imputation methods over six real-world incomplete datasets.
Some results are unavailable (represented by ``$-$''), since the corresponding methods are not able to finish within $10^5$ seconds.

One can observe that, \textsf{SCIS} takes less training time and smaller training sample rate $R_t$ than baselines,
while it achieves competitive imputation accuracy (i.e., similar RMSE value).
Moreover, \textsf{SCIS}-GAIN outperforms the other methods in most cases.
Specifically, \textsf{SCIS} adopts only 7.58\% training samples and saves 41.75\% training time of GAN-based methods in average.
For the last three \emph{million-size} incomplete datasets in Table~\ref{Tab:all_imputation2}, \textsf{SCIS} takes only 1.53\% training samples and saves 86.85\% training time of the GAN-based methods in average.
In particular, compared with the Sinkhorn divergence based method RRSI, \textsf{SCIS}-GAIN takes only 17.94\% training samples and saves 86.65\% training time in average, while averagely increases 3.74\% accuracy.
We omit MissF, Baran, MICE, RRSI, MIDAE, VAEI, MIWAE, and EDDI on the larger million-size datasets, due to the high model complexity.

Moreover, the speedup of \textsf{SCIS} is not very obvious on the first three dataset in Table~\ref{Tab:all_imputation1}, compared with the other ones in Table~\ref{Tab:all_imputation2}.
It is because that, the initial sample size $n_0$ for the first three datasets are significantly smaller than that of the other ones.
The smaller $n_0$, the larger the variance derived by Eq. \ref{eq:conditional}, resulting in higher training sample rate and training time.
The training time (resp. training sample rate) even decreases to 4.52\% (resp. 0.67\%) on \emph{Search} (resp. \emph{Surveil}) for GAIN.
It is attributed to the SSE module in \textsf{SCIS} that minimizes the required training sample size of GAN-based methods.
In addition, the competitive (even better) accuracy with (than) original methods results from the MS divergence imputation loss function employed in \textsf{SCIS} that measures the closeness between the true underlying data and generated data distributions.
Also, the accuracy guarantee in SSE benefits the accuracy of \textsf{SCIS}.
It even increases 3.02\% accuracy for GAIN on \emph{Trial}.
In particular, the experimental results of \textsf{SCIS}-GINN are unavailable over \emph{Search} and \emph{Surveil} datasets, since GINN has a high complexity on construction of the similarity graph.
For further extensive experimental study on \textsf{SCIS}, we employ GAIN as the baseline, since GAIN can work on these datasets, providing a clear comparison benchmark.

\begin{figure*}[t]
\centering
\includegraphics[width=0.45\linewidth]{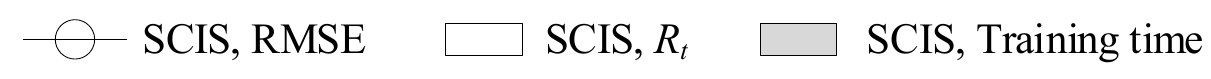}\vspace*{-0.03in}
\\
\hspace*{-0.1in}
\subfigure[\emph{Trial}]{
\raisebox{-2cm}{\includegraphics[width=0.315\linewidth]{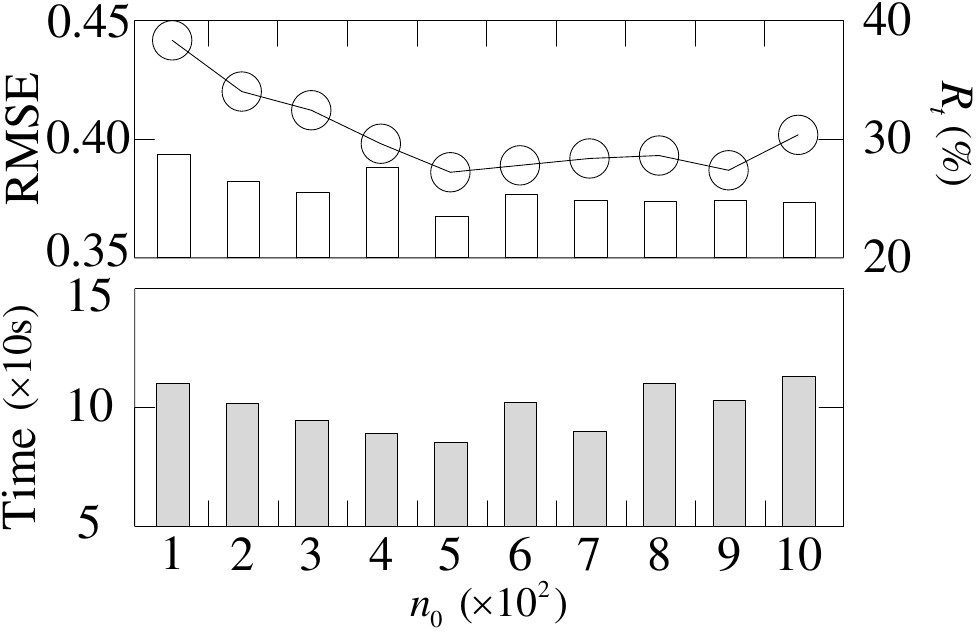}}}\hspace*{0.05in}
\subfigure[\emph{Emergency}]{
\raisebox{-2cm}{\includegraphics[width=0.315\linewidth]{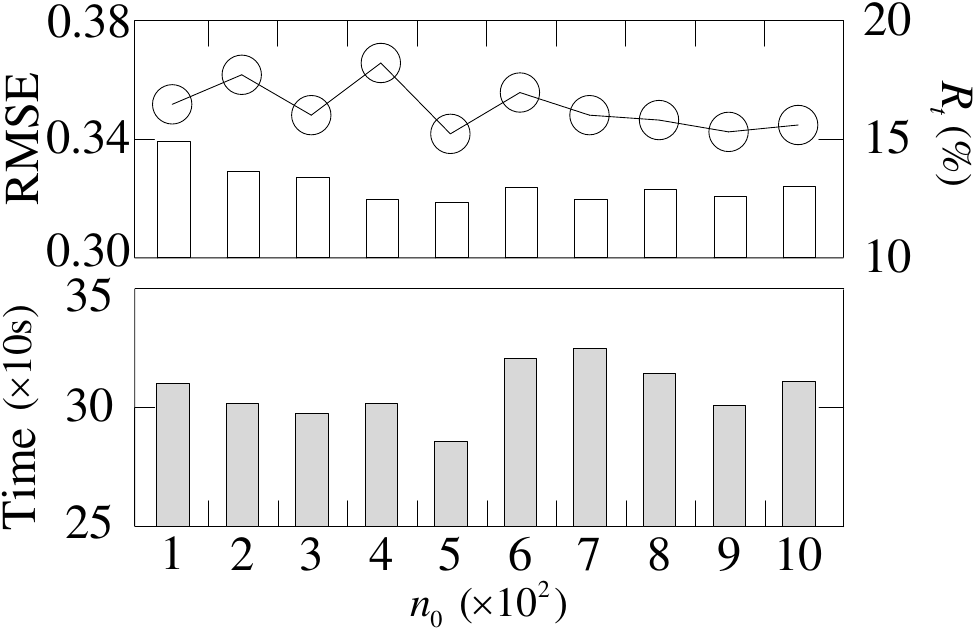}}}\hspace*{0.05in}
\subfigure[\emph{Response}]{
\raisebox{-2cm}{\includegraphics[width=0.315\linewidth]{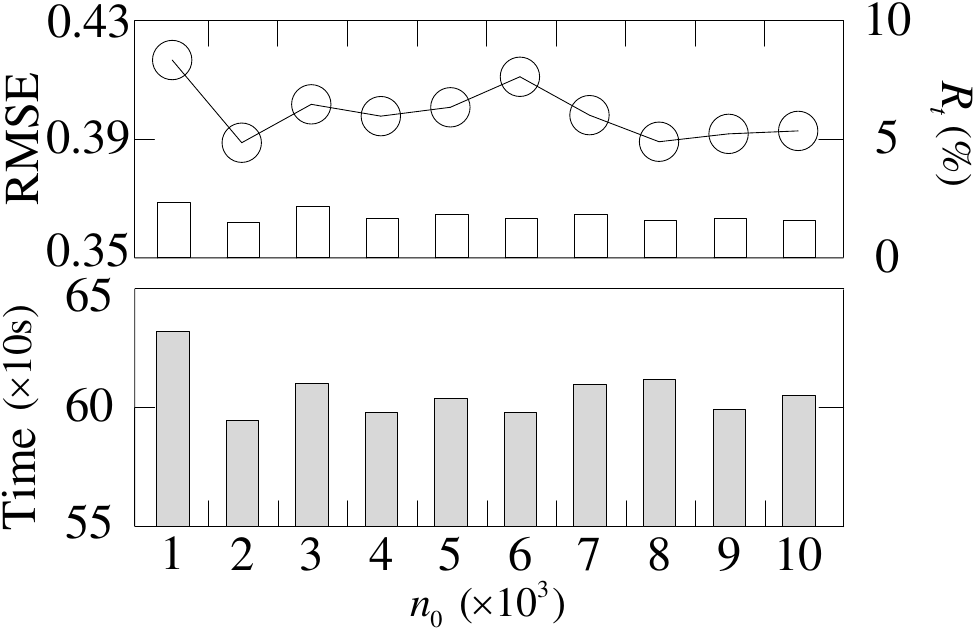}}}
\vspace*{-0.1in}
\\\hspace*{-0.14in}
\subfigure[\emph{Search}]{
\raisebox{-2cm}{\includegraphics[width=0.315\linewidth]{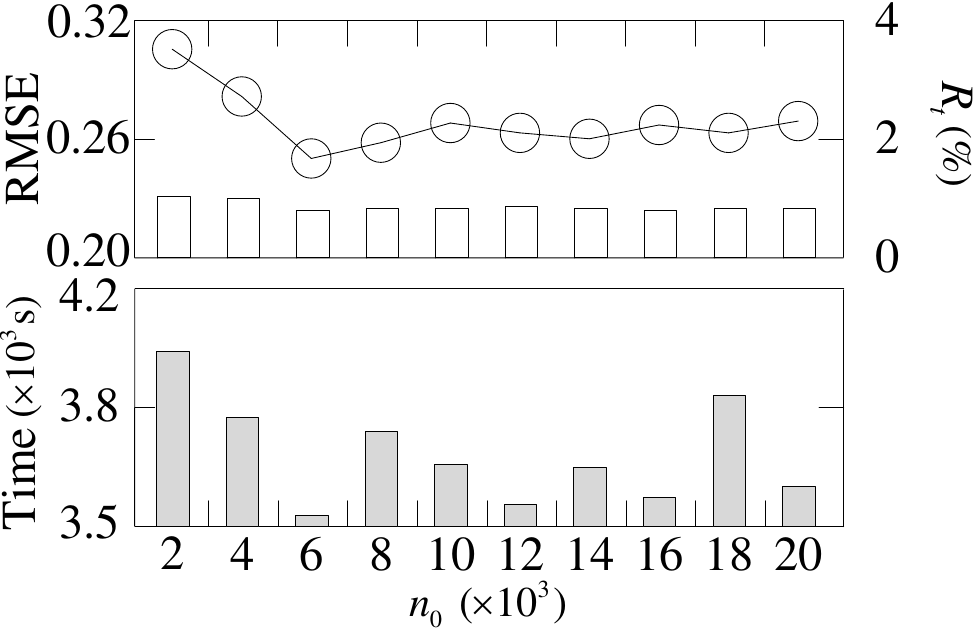}}}\hspace*{0.05in}
\subfigure[\emph{Weather}]{
\raisebox{-2cm}{\includegraphics[width=0.315\linewidth]{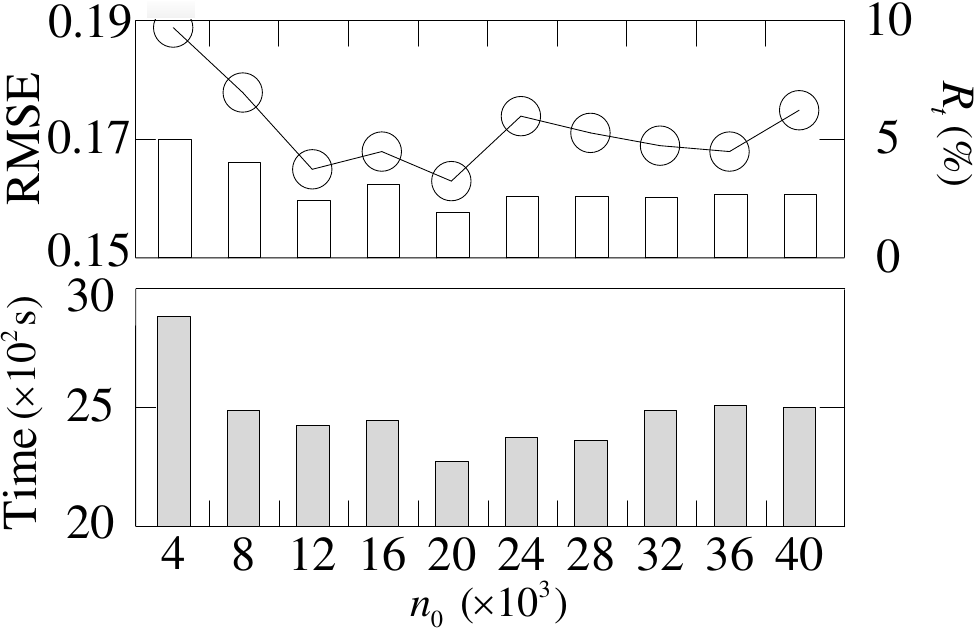}}}\hspace*{0.05in}
\subfigure[\emph{Surveil}]{
\raisebox{-2cm}{\includegraphics[width=0.315\linewidth]{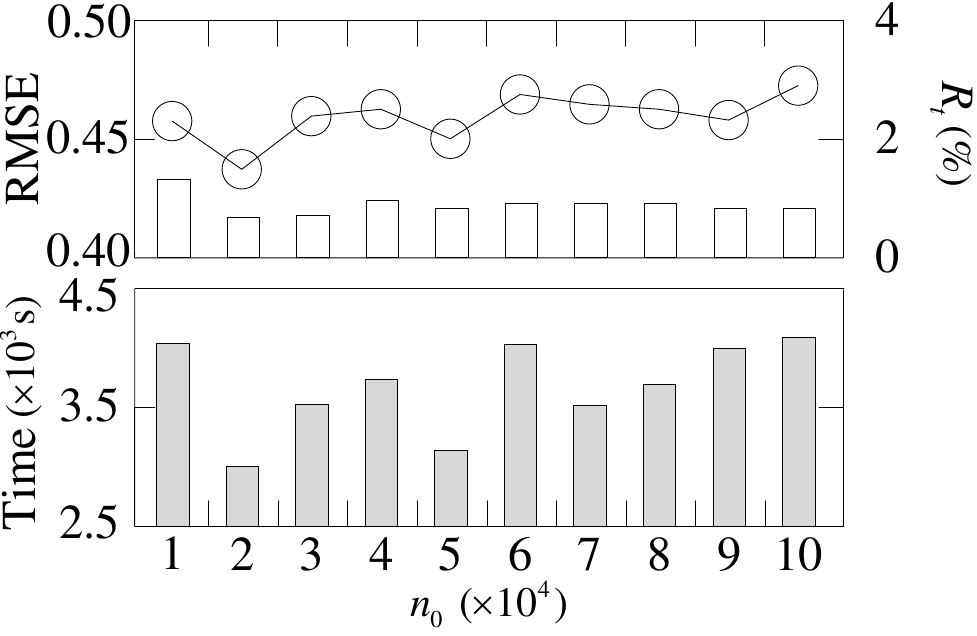}}}\\
\vspace*{-0.1in}
\caption{The performance of \textsf{SCIS} vs. $n_0$}
\label{fig:initial_sizes}
\end{figure*}

\subsection{Parameter Evaluation}

\textbf{Effect of $R_m$.}
When varying the missing rate $R_m$ (i.e., how many values in original observed data are dropped) from 10\% to 90\%, the corresponding results are depicted in Figure~\ref{fig:MR}. It also reports the time cost of the SSE module, which is the core module of \textsf{SCIS}.
The reported \textsf{SCIS} training time has included the execution time of SSE.
We can find that, compared with GAIN, \textsf{SCIS} takes much less training time and training samples to obtain a similar or even higher imputation accuracy in all cases.
It is more robust with the increasing missing rate $R_m$ than GAIN.
The SSE module takes 28.31\% training time of \textsf{SCIS} in average.
In addition, the imputation accuracy of both GAIN and \textsf{SCIS} is comparable, and it descends consistently with the growth of missing rate.
The reason is that, as the missing rate increases, the observed information for algorithms becomes less, making imputation algorithms less effective.

\textbf{Effect of $\varepsilon$.}
To verify the effectiveness of the imputation accuracy guarantee, we study the effect of the user-tolerated error bound $\varepsilon$ on the performance of \textsf{SCIS} with GAIN.
With $\varepsilon$ varying from 0.001 to 0.009, Figure \ref{fig:Varepsilon} plots the corresponding results.
The sample rate of the initial training set $\mathbf{X}_0$, i.e., $R_1 = n_0/N$ (and that of the minimum sample set $\mathbf{X}_\star$, i.e., $R_2 = n_\star/N$).
The user-tolerated error is derived by $R^u_{mse}$ + $\varepsilon$, where $R^u_{mse}$ is the RMSE value of GAIN with the MS divergence imputation loss function trained on $\mathbf{X}$.
In order to verify the effect of the DIM module on \textsf{SCIS}, we also report the imputation error $R^o_{mse}$ + $\varepsilon$ in the figure, where $R^o_{mse}$ is the RMSE value of original GAIN model trained on $\mathbf{X}$.
Some results are unavailable in the figure, because the corresponding methods cannot finish within $10^5$ seconds.

As inferred from the figure, \textsf{SCIS} has the higher imputation accuracy than the user-tolerated error $R^u_{mse}$ + $\varepsilon$ and the imputation error $R^o_{mse}$ + $\varepsilon$ in most cases.
It means that, the SSE module can estimate an appropriate sample size to get as the good imputation accuracy as the users' expectation, i.e., it indeed satisfies the accuracy requirement of users.
Besides, the error $R^u_{mse}$ derived by \textsf{SCIS} is smaller than $R^o_{mse}$ derived by GAIN in many cases.
It confirms that, the DIM module using the MS divergence does boost the imputation accuracy.
Moreover, in most cases, the RMSE of \textsf{SCIS} increases with the growth of $\varepsilon$, while $R_2$ is opposite.
It is because, a smaller value of $\varepsilon$ signifies a lower user-tolerated error.
Hence, more samples (i.e., a larger $R_2$) are needed to satisfy a lower RMSE requirement.
In addition, when $\varepsilon$ exceeds 0.005, the RMSE of \textsf{SCIS} changes slightly since $n_\star$ equals the lower bound $n_0$.

\textbf{Effect of $n_0$.}
Figure~\ref{fig:initial_sizes} depicts the experimental results of varying the initial sample size $n_0$.
For \textsf{SCIS} with GAIN, different datasets require different optimal $n_0$. \textsf{SCIS} achieves the best imputation accuracy (in terms of RMSE) when the optimal $n_0$ is chosen, i.e., 500 for \emph{Trial}, 500 for \emph{Emergency}, 2,000 for \emph{Response}, 6,000 for \emph{Search}, 20,000 for \emph{Weather}, and 20,000 for \emph{Surveil}. Meanwhile, the time consumption and training sample rate remain reasonable and acceptable.
In addition, the $R_t$ of \textsf{SCIS} increases with the decrease of $n_0$ in most cases.
It is partially because that, the less the initial sample size $n_0$, the larger the variance derived by Eq. \ref{eq:conditional}, leading to more training samples.

\begin{table*}[t]\small
\centering
\setlength{\tabcolsep}{4.5pt}
\caption{The ablation study of \textsf{SCIS} over Trial, Emergency, and Response}
\vspace*{-0.08in}
\label{Tab:ablation1}
\begin{tabular}{ccccccccccccccccccc}
\hline
\multirow{2}{*}{Method} &\multicolumn{3}{c}{\emph{Trial}}&\multicolumn{3}{c}{\emph{Emergency}}&\multicolumn{3}{c}{\emph{Response}}\\\cline{2-10}
 & RMSE (Bias)& Time (s) &$R_t$ (\%)& RMSE (Bias)& Time (s) &$R_t$ (\%)& RMSE (Bias)& Time (s) &$R_t$ (\%)\\ \hline
GAIN &0.398 ($\pm$ 0.024)&90&100&0.352 ($\pm$ 0.025)&340&100&0.396 ($\pm$ 0.031)&649&100\\
DIM-GAIN &\textbf{0.383 ($\pm$ 0.022)}&292&100 &\textbf{0.340 ($\pm$ 0.025)}&617&100 &\textbf{0.386 ($\pm$ 0.028)}&7,771&100\\
Fixed-DIM-GAIN & 0.389 ($\pm$ 0.025)& \textbf{57}&\textbf{10} &0.348 ($\pm$ 0.028)&\textbf{125}&\textbf{10} &0.388 ($\pm$ 0.002)&836&10\\
\textsf{SCIS}-GAIN &0.386 ($\pm$ 0.017)&85&23.56&0.342 ($\pm$ 0.019)&286&12.32&0.389 ($\pm$ 0.021)&\textbf{595}&\textbf{1.50}\\\hline
\end{tabular}
\vspace*{0.05in}
\end{table*}

\begin{table*}[t]\small
\centering
\setlength{\tabcolsep}{4.5pt}
\caption{The ablation study  of \textsf{SCIS} over Search, Weather, and Surveil}
\vspace*{-0.08in}
\label{Tab:ablation2}
\begin{tabular}{ccccccccccccccccccc}
\hline
\multirow{2}{*}{Method} &\multicolumn{3}{c}{\emph{Search}}&\multicolumn{3}{c}{\emph{Weather}} &\multicolumn{3}{c}{\emph{Surveil}}\\\cline{2-10}
 & RMSE (Bias)& Time (s) &$R_t$ (\%)& RMSE (Bias)& Time (s) &$R_t$ (\%)& RMSE (Bias)& Time (s) &$R_t$ (\%)\\ \hline
GAIN & 0.252 ($\pm$ 0.014)&78,121&100 &0.165 ($\pm$ 0.014)&9,252&100 &0.440 ($\pm$ 0.012)&29,102&100\\
DIM-GAIN & $-$ &$-$&$-$ & $-$ &$-$&$-$ & $-$&$-$&$-$\\
Fixed-DIM-GAIN &\textbf{0.247 ($\pm$ 0.007)}&27,045 &10 &\textbf{0.158 ($\pm$ 0.012)}&18,408&10 &\textbf{0.434 ($\pm$ 0.015)}&78,508&10\\
\textsf{SCIS}-GAIN &0.250 ($\pm$ 0.013)& \textbf{3,534}&\textbf{0.78}&0.163 ($\pm$ 0.016)&\textbf{2,275}&\textbf{1.90}&0.438 ($\pm$ 0.013)&\textbf{3,019}&\textbf{0.67}\\\hline
\end{tabular}
\vspace*{0.05in}
\end{table*}

\begin{table}[t]\small
\centering
\setlength{\tabcolsep}{6pt}
\caption{Results of post-imputation prediction}
\vspace*{-0.08in}
\label{Tab:Post-imputation}
\begin{tabular}{|c|c|c|c|}
\hline
Metric&Dataset& GAIN& \textsf{SCIS}-GAIN\\ \hline
\multirow{2}{*}{AUC} &\emph{Trial} &0.903 ($\pm$ 0.040)&\textbf{0.905 ($\pm$ 0.040)}\\\cline{2-4}
&\emph{Surveil} &0.949 ($\pm$ 0.020)&\textbf{0.952 ($\pm$ 0.025)}\\\hline
\multirow{4}{*}{MAE}
&\emph{Emergency} & 122.312 ($\pm$ 5.322)&\textbf{121.683 ($\pm$ 4.985)}\\\cline{2-4}
&\emph{Response} & 106.216 ($\pm$ 7.689)& \textbf{105.343 ($\pm$ 6.321)}\\\cline{2-4}
&\emph{Search} &89.142 ($\pm$ 9.543)& \textbf{88.879 ($\pm$ 9.901)}\\\cline{2-4}
&\emph{Weather} &100.262 ($\pm$ 7.837)&\textbf{99.872 ($\pm$ 9.222)}\\\hline
\end{tabular}
\end{table}

\subsection{Ablation Study}
We investigate the influence of different modules of \textsf{SCIS} on the imputation performance.
The corresponding experimental results of the RMSE, training time, and training sample rate $R_t$ are shown in Table \ref{Tab:ablation1} and Table \ref{Tab:ablation2}.
DIM-GAIN is the variant of \textsf{SCIS}-GAIN without the SSE module over GAIN.
Fixed-DIM-GAIN, a variant of DIM-GAIN, randomly selects ten percentage of samples as the training data to accelerate the model training process.
In Table \ref{Tab:ablation2}, the results of DIM-GAIN are unavailable (represented by ``$-$''), since they are not able to finish within $10^5$ seconds.

We can observe that, DIM-GAIN gains better imputation accuracy (i.e., smaller RMSE value) than GAIN, while requires higher training time.
Besides, compared with DIM-GAIN and Fixed-DIM-GAIN, \textsf{SCIS}-GAIN takes significantly less training time and training samples, while shows a negligible decrease in imputation accuracy, especially on the last three incomplete datasets.
Specifically, DIM-GAIN increases 3.24\% accuracy for GAIN in average, but its training time is 4.68x averagely of GAIN.
It confirms the effectiveness of the MS divergence imputation loss function in the DIM module.
\textsf{SCIS}-GAIN takes 6.79\% (resp. 67.88\%) training samples and saves 72.29\% (resp. 20.27\%) training time of DIM-GAIN (resp. Fixed-DIM-GAIN) in average, while only averagely decreases 0.72\% (resp. 0.51\%)  accuracy for DIM-GAIN (resp. Fixed-DIM-GAIN).
In general, these results further confirm that the sample size estimation strategy in the SSE module is valid and indispensable.
Moreover, for Fixed-DIM-GAIN and \textsf{SCIS}-GAIN that use the same imputation model (i.e., DIM-GAIN),
the more training samples (i.e., higher training sample rate $R_t$) the imputation method requires, the higher imputation accuracy it achieves.
It is because that, the more the training samples, the higher the available information for imputation algorithms, making algorithms more powerful.

\subsection{Evaluation on Post-imputation Prediction}
\label{sec:appendix-results}

The ultimate goal of imputing missing data is to benefit the \emph{downstream data analytics}, e.g., regression and classification.
In the last set of experiments, we verify the superiority of \textsf{SCIS} over GAIN on the \emph{post-imputation prediction} task.

For the classification task over \emph{Trial} and \emph{Surveil} and the regression task over \emph{Emergency}, \emph{Response}, \emph{Search}, and \emph{Weather}, the corresponding prediction results are depicted in Table \ref{Tab:Post-imputation}.
The larger AUC value corresponds to the better prediction effect, while MAE is opposite.
The imputation methods are first employed to impute missing values in the incomplete datasets.
Then, a regression/classification model is trained with three fully-connected layers over the imputed data.
The training epoch is 30, the learning rate is 0.005, the dropout rate is 0.5, and the batch size is 128.

We can observe that, the prediction performance under different imputation algorithms is consistent with the imputation performance of these algorithms, i.e., \textsf{SCIS}-GAIN gains competitive (even better) accuracy with (than) GAIN.
Specifically, \textsf{SCIS}-GAIN decreases 0.51\% MAE for GAIN on the regression task, while increases 0.27\% AUC for GAIN on the classification task.
In addition, on the regression (resp. classification) task, \textsf{SCIS}-GAIN achieves the larger improvement over the \emph{Weather} (resp. \emph{Response}) dataset.
Thus, it further confirms the effectiveness of \textsf{SCIS}.

%% file: 7.conclusion.tex
\section{Conclusion}
\label{sec:conclusion}

In this paper, we propose an effective scalable imputation system \textsf{SCIS} to accelerate GAN-based imputation models.
It consists of a DIM module and an SSE module.
DIM makes the generative adversarial imputation models \emph{differentiable} via leveraging a new MS divergence.
SSE estimates the \emph{minimum} training sample size for the differentiable imputation model, so that the final trained model satisfies a user-tolerated error bound.
Extensive experiments over several real-world datasets demonstrate that, \textsf{SCIS} significantly accelerates the model training and meanwhile harvests the competitive imputation accuracy with the state-of-the-art GAN-based methods.

Almost all existing GAN-based imputation algorithms are under the simple MCAR assumption.
It to some extent limits the effectiveness of \textsf{SCIS} under the real complex missing mechanism.
In the future, we intend to further explore the imputation problem under more complex missing mechanisms for large-scale incomplete data.

%% file: SCIS.bbl
\begin{thebibliography}{10}
\providecommand{\url}[1]{#1}
\csname url@samestyle\endcsname
\providecommand{\newblock}{\relax}
\providecommand{\bibinfo}[2]{#2}
\providecommand{\BIBentrySTDinterwordspacing}{\spaceskip=0pt\relax}
\providecommand{\BIBentryALTinterwordstretchfactor}{4}
\providecommand{\BIBentryALTinterwordspacing}{\spaceskip=\fontdimen2\font plus
\BIBentryALTinterwordstretchfactor\fontdimen3\font minus
  \fontdimen4\font\relax}
\providecommand{\BIBforeignlanguage}[2]{{%
\expandafter\ifx\csname l@#1\endcsname\relax
\typeout{** WARNING: IEEEtran.bst: No hyphenation pattern has been}%
\typeout{** loaded for the language `#1'. Using the pattern for}%
\typeout{** the default language instead.}%
\else
\language=\csname l@#1\endcsname
\fi
#2}}
\providecommand{\BIBdecl}{\relax}
\BIBdecl

\bibitem{berti2018discovery}
L.~Berti-Equille, H.~Harmouch, F.~Naumann, N.~Novelli, and
  S.~Thirumuruganathan, ``Discovery of genuine functional dependencies from
  relational data with missing values,'' \emph{Proceedings of the VLDB
  Endowment}, vol.~11, no.~8, pp. 880--892, 2018.

\bibitem{miao2021generative}
X.~Miao, Y.~Wu, J.~Wang, Y.~Gao, X.~Mao, and J.~Yin, ``Generative
  semi-supervised learning for multivariate time series imputation,'' in
  \emph{AAAI}, 2021, pp. 8983--8991.

\bibitem{rekatsinas2017holoclean}
T.~Rekatsinas, X.~Chu, I.~F. Ilyas, and C.~R{\'e}, ``{HoloClean: Holistic data
  repairs with probabilistic inference},'' \emph{Proceedings of the VLDB
  Endowment}, vol.~10, no.~11, pp. 1190--1201, 2017.

\bibitem{soliman2010supporting_35}
M.~A. Soliman, I.~F. Ilyas, and S.~Ben-David, ``Supporting ranking queries on
  uncertain and incomplete data,'' \emph{The VLDB Journal}, vol.~19, no.~4, pp.
  477--501, 2010.

\bibitem{wei2019embedded}
Z.~Wei and S.~Link, ``Embedded functional dependencies and data-completeness
  tailored database design,'' \emph{Proceedings of the VLDB Endowment},
  vol.~12, no.~11, pp. 1458--1470, 2019.

\bibitem{biessmann2018deep}
F.~Biessmann, D.~Salinas, S.~Schelter, P.~Schmidt, and D.~Lange, ``Deep
  learning for missing value imputationin tables with non-numerical data,'' in
  \emph{CIKM}, 2018, pp. 2017--2025.

\bibitem{cao2018brits}
W.~Cao, D.~Wang, J.~Li, H.~Zhou, L.~Li, and Y.~Li, ``{BRITS: B}idirectional
  recurrent imputation for time series,'' in \emph{NeurIPS}, 2018, pp.
  6775--6785.

\bibitem{royston2011multiple}
P.~Royston and I.~R. White, ``{Multiple imputation by chained equations (MICE):
  Implementation in Stata},'' \emph{Journal of Statistical Software}, vol.~45,
  no.~4, pp. 1--20, 2011.

\bibitem{mattei2019miwae}
P.-A. Mattei and J.~Frellsen, ``{MIWAE}: Deep generative modelling and
  imputation of incomplete data sets,'' in \emph{ICML}, 2019, pp. 4413--4423.

\bibitem{farhangfar2007novel}
A.~Farhangfar, L.~A. Kurgan, and W.~Pedrycz, ``A novel framework for imputation
  of missing values in databases,'' \emph{IEEE Transactions on Systems, Man,
  and Cybernetics -Part A: Systems and Humans}, vol.~37, no.~5, pp. 692--709,
  2007.

\bibitem{stekhoven2011missforest}
D.~J. Stekhoven and P.~B{\"u}hlmann, ``{MissForest non parametric missing value
  imputation for mixed-type data},'' \emph{Bioinformatics}, vol.~28, no.~1, pp.
  112--118, 2011.

\bibitem{boris2020missing}
M.~Boris, J.~Julie, B.~Claire, and C.~Marco, ``Missing data imputation using
  optimal transport,'' in \emph{ICML}, 2020, pp. 1--18.

\bibitem{ipsen2020not}
N.~B. Ipsen, P.-A. Mattei, and J.~Frellsen, ``{not-MIWAE}: Deep generative
  modelling with missing not at random data,'' \emph{ArXiv Preprint
  ArXiv:2006.12871}, 2020.

\bibitem{yoon2018gain}
J.~Yoon, J.~Jordon, and M.~van~der Schaar, ``{GAIN}: Missing data imputation
  using generative adversarial nets,'' in \emph{ICML}, 2018, pp. 5675--5684.

\bibitem{spinelli2019missing}
I.~Spinelli, S.~Scardapane, and A.~Uncini, ``Missing data imputation with
  adversarially-trained graph convolutional networks,'' \emph{ArXiv Preprint
  ArXiv:1905.01907}, 2019.

\bibitem{kim2020survey}
J.~Kim, D.~Tae, and J.~Seok, ``A survey of missing data imputation using
  generative adversarial networks,'' in \emph{ICAIIC}, 2020, pp. 454--456.

\bibitem{Wahltinez2020}
\BIBentryALTinterwordspacing
O.~Wahltinez \emph{et~al.}, ``{COVID-19 Open-Data}: Curating a fine-grained,
  global-scale data repository for {SARS-CoV-2},'' 2020, work in Progress.
  [Online]. Available: \url{https://goo.gle/covid-19-open-data}
\BIBentrySTDinterwordspacing

\bibitem{arjovsky2017towards}
M.~Arjovsky and L.~Bottou, ``Towards principled methods for training generative
  adversarial networks,'' \emph{ArXiv Preprint ArXiv:1701.04862}, 2017.

\bibitem{bellemare2017cramer}
M.~G. Bellemare, I.~Danihelka, W.~Dabney, S.~Mohamed, B.~Lakshminarayanan,
  S.~Hoyer, and R.~Munos, ``The cramer distance as a solution to biased
  wasserstein gradients,'' \emph{ArXiv Preprint ArXiv:1705.10743}, 2017.

\bibitem{arjovsky2017wasserstein}
M.~Arjovsky, S.~Chintala, and L.~Bottou, ``Wasserstein generative adversarial
  networks,'' in \emph{ICML}, 2017, pp. 214--223.

\bibitem{twala2005comparison}
B.~Twala, M.~Cartwright, and M.~Shepperd, ``Comparison of various methods for
  handling incomplete data in software engineering databases,'' in \emph{ESEM},
  2005, pp. 234--239.

\bibitem{altman1992introduction}
N.~S. Altman, ``An introduction to kernel and nearest -neighbor nonparametric
  regression,'' \emph{The American Statistician}, vol.~46, no.~3, pp. 175--185,
  1992.

\bibitem{zhang2021parrot}
Y.~Zhang, H.~Zhang, Z.~He, Y.~Jing, K.~Zhang, and X.~S. Wang, ``Parrot: A
  progressive analysis system on large text collections,'' \emph{Data Science
  and Engineering}, vol.~6, no.~1, pp. 1--19, 2021.

\bibitem{yang2021keyword}
J.~Yang, W.~Yao, and W.~Zhang, ``Keyword search on large graphs: A survey,''
  \emph{Data Science and Engineering}, vol.~6, no.~2, pp. 142--162, 2021.

\bibitem{chen2016xgboost}
T.~Chen and C.~Guestrin, ``{XGBoost}: A scalable tree boosting system,'' in
  \emph{SIGKDD}, 2016, pp. 785--794.

\bibitem{mahdavi2020baran}
M.~Mahdavi and Z.~Abedjan, ``Baran: Effective error correction via a unified
  context representation and transfer learning,'' \emph{Proceedings of the VLDB
  Endowment}, vol.~13, no.~12, pp. 1948--1961, 2020.

\bibitem{zhang2019learning}
A.~Zhang, S.~Song, Y.~Sun, and J.~Wang, ``Learning individual models for
  imputation,'' in \emph{ICDE}, 2019, pp. 160--171.

\bibitem{ruder2016overview}
S.~Ruder, ``An overview of gradient descent optimization algorithms,''
  \emph{ArXiv Preprint ArXiv:1609.04747}, 2016.

\bibitem{du2021deep}
G.~Du, L.~Zhou, Y.~Yang, K.~L{\"u}, and L.~Wang, ``Deep multiple
  auto-encoder-based multi-view clustering,'' \emph{Data Science and
  Engineering}, pp. 1--16, 2021.

\bibitem{mahajan2021predicting}
R.~Mahajan and V.~Mansotra, ``Predicting geolocation of tweets: Using
  combination of cnn and bilstm,'' \emph{Data Science and Engineering}, vol.~6,
  no.~4, pp. 402--410, 2021.

\bibitem{biessmann2019datawig}
F.~Biessmann, T.~Rukat, P.~Schmidt, P.~Naidu, S.~Schelter, A.~Taptunov,
  D.~Lange, and D.~Salinas, ``Datawig: Missing value imputation for tables,''
  \emph{Journal of Machine Learning Research}, vol.~20, pp. 175--1, 2019.

\bibitem{gondara2017multiple}
L.~Gondara and K.~Wang, ``Multiple imputation using deep denoising
  autoencoders,'' \emph{ArXiv Preprint ArXiv:1705.02737}, 2017.

\bibitem{mccoy2018variational}
J.~T. McCoy, S.~Kroon, and L.~Auret, ``Variational autoencoders for missing
  data imputation with application to a simulated milling circuit,''
  \emph{IFAC-PapersOnLine}, vol.~51, no.~21, pp. 141--146, 2018.

\bibitem{ma2018eddi}
C.~Ma, S.~Tschiatschek, K.~Palla, J.~M. Hern{\'a}ndez-Lobato, S.~Nowozin, and
  C.~Zhang, ``{EDDI: Efficient} dynamic discovery of high-value information
  with partial vae,'' \emph{ArXiv Preprint ArXiv:1809.11142}, 2018.

\bibitem{nazabal2018handling}
A.~Nazabal, P.~M. Olmos, Z.~Ghahramani, and I.~Valera, ``{Handling incomplete
  heterogeneous data using VAEs},'' \emph{ArXiv Preprint ArXiv:1807.03653},
  2018.

\bibitem{peyre2019computational}
G.~Peyr{\'e}, M.~Cuturi \emph{et~al.}, ``Computational optimal transport: With
  applications to data science,'' \emph{Foundations and Trends{\textregistered}
  in Machine Learning}, vol.~11, no. 5-6, pp. 355--607, 2019.

\bibitem{genevay2016stochastic}
A.~Genevay, M.~Cuturi, G.~Peyr{\'e}, and F.~Bach, ``Stochastic optimization for
  large-scale optimal transport,'' in \emph{NeurIPS}, 2016, pp. 3440--3448.

\bibitem{sinkhorn1964relationship}
R.~Sinkhorn, ``A relationship between arbitrary positive matrices and doubly
  stochastic matrices,'' \emph{The Annals of Mathematical Statistics}, vol.~35,
  no.~2, pp. 876--879, 1964.

\bibitem{cuturi2014fast}
M.~Cuturi and A.~Doucet, ``Fast computation of {W}asserstein barycenters,'' in
  \emph{ICML}, 2014, pp. 685--693.

\bibitem{yu2017seqgan}
L.~Yu, W.~Zhang, J.~Wang, and Y.~Yu, ``{SeqGAN}: Sequence generative
  adversarial nets with policy gradient,'' in \emph{AAAI}, 2017, pp.
  2852--2858.

\bibitem{salimans2018improving}
T.~Salimans, H.~Zhang, A.~Radford, and D.~Metaxas, ``Improving gans using
  optimal transport,'' \emph{ArXiv Preprint ArXiv:1803.05573}, 2018.

\bibitem{newey1994large}
W.~K. Newey and D.~McFadden, ``Large sample estimation and hypothesis
  testing,'' \emph{Handbook of Econometrics}, vol.~4, pp. 2111--2245, 1994.

\bibitem{park2019blinkml}
Y.~Park, J.~Qing, X.~Shen, and B.~Mozafari, ``{BlinkML}: Efficient maximum
  likelihood estimation with probabilistic guarantees,'' in \emph{SIGMOD},
  2019, pp. 1135--1152.

\bibitem{genevay2019sample}
A.~Genevay, L.~Chizat, F.~Bach, M.~Cuturi, and G.~Peyr{\'e}, ``Sample
  complexity of sinkhorn divergences,'' in \emph{AISTATS}, 2019, pp.
  1574--1583.

\bibitem{nilsen2019efficient}
G.~K. Nilsen, A.~Z. Munthe-Kaas, H.~J. Skaug, and M.~Brun, ``Efficient
  computation of hessian matrices in tensorflow,'' \emph{ArXiv Preprint
  ArXiv:1905.05559}, 2019.

\bibitem{hale2021global}
T.~Hale, N.~Angrist, R.~Goldszmidt, B.~Kira, A.~Petherick, T.~Phillips,
  S.~Webster, E.~Cameron-Blake, L.~Hallas, S.~Majumdar \emph{et~al.}, ``A
  global panel database of pandemic policies (oxford covid-19 government
  response tracker),'' \emph{Nature Human Behaviour}, vol.~5, no.~4, pp.
  529--538, 2021.

\end{thebibliography}
